\documentclass{article}
\usepackage{amsmath}
\usepackage{amsthm}
\usepackage{amsfonts}
\usepackage{bm}

\usepackage{amssymb}
\usepackage{listings}
\usepackage{mathtools}
\usepackage{microtype}
\usepackage{graphicx}
\usepackage{subfigure}
\usepackage{booktabs} %
\usepackage{selectp}
\usepackage{hyperref}
\usepackage{wrapfig}

\newcommand{\D}{\mathrm{D}}
\newcommand{\KL}{\mathrm{KL}}
\newcommand{\BE}{\mathbb{E}}
\newcommand{\mc}{\mathcal}

\newtheorem{theorem}{Theorem}
\newtheorem{lemma}{Lemma}

\newtheorem{assumption}{Assumption}
\newtheorem{definition}{Definition}
\newtheorem{example}{Example}

\newcommand{\norm}[1]{\left\lVert#1\right\rVert}

\usepackage[accepted]{icml2022}

\icmltitlerunning{Offline Imitation from Observations, Mismatched Experts, and Examples}

\newcommand{\para}[1]{\textbf{#1}}

\usepackage{ulem}

\begin{document}

\twocolumn[
\icmltitle{Versatile Offline Imitation from Observations and Examples via\\ Regularized State-Occupancy Matching}

\begin{icmlauthorlist}
\icmlauthor{Yecheng Jason Ma}{penn}
\icmlauthor{Andrew Shen}{melbourne}
\icmlauthor{Dinesh Jayaraman}{penn}
\icmlauthor{Osbert Bastani}{penn}
\end{icmlauthorlist}

\icmlaffiliation{penn}{Department of Computer and Information Science, University of Pennsylvania, Philadelphia, USA}
\icmlaffiliation{melbourne}{University of Melbourne, Melbourne, Australia}

\icmlcorrespondingauthor{Yecheng Jason Ma}{jasonyma@seas.upenn.edu}

\icmlkeywords{Machine Learning, ICML}

\vskip 0.3in
]

\printAffiliationsAndNotice{}  %

\begin{abstract}
We propose \textbf{S}tate \textbf{M}atching \textbf{O}ffline \textbf{DI}stribution \textbf{C}orrection \textbf{E}stimation (SMODICE), a novel and versatile regression-based offline imitation learning (IL) algorithm derived via state-occupancy matching. We show that the SMODICE objective admits a simple optimization procedure through an application of Fenchel duality and an analytic solution in tabular MDPs. Without requiring access to expert actions, SMODICE can be effectively applied to three offline IL settings: (i) imitation from observations (IfO), (ii) IfO with dynamics or morphologically mismatched expert, and (iii) example-based reinforcement learning, which we show can be formulated as a state-occupancy matching problem. We extensively evaluate SMODICE on both gridworld environments as well as on high-dimensional offline benchmarks. Our results demonstrate that SMODICE is effective for all three problem settings and significantly outperforms prior state-of-art. Project website: \href{https://sites.google.com/view/smodice/home}{https://sites.google.com/view/smodice/home}
\end{abstract}

\section{Introduction}

The offline reinforcement learning (RL) framework ~\cite{lange2012batch, levine2020offline} aims to use pre-collected, reusable offline data---without further interaction with the environment---for sample-efficient, scalable, and practical data-driven decision-making. 
However, this %
assumes that the offline dataset comes with reward labels, which may not always be possible.
To address this, offline \textit{imitation} learning (IL)~\cite{zolna2020offline, chang2021mitigating, kim2022demodice} has recently been proposed as an alternative where the learning algorithm is 
provided with a small set of expert demonstrations and a separate set of offline data of unknown quality. The goal is to learn a policy that mimics the provided expert data while avoiding test-time distribution shift~\cite{ross2011reduction} by using the offline dataset.

Expert demonstrations are often much more expensive to acquire than offline data; thus, offline IL benefits significantly from minimizing assumptions about the expert data. 
In this work, we aim to remove two assumptions about the expert data in 
current offline IL algorithms: (i) expert action labels must be provided for the demonstrations, and (ii) the expert demonstrations are performed with identical dynamics (same embodiment, actions, and transitions) as the imitator agent.
These requirements preclude applications to important practical problem settings, including (i) imitation from observations, (ii) imitation with mismatched expert that obeys different dynamics or embodiment (e.g., learning from human videos), and (iii) learning only from examples of successful outcomes rather than full expert trajectories~\cite{eysenbach2021replacing}. 
\begin{figure*}[t!]
\centering
\includegraphics[width=\textwidth]{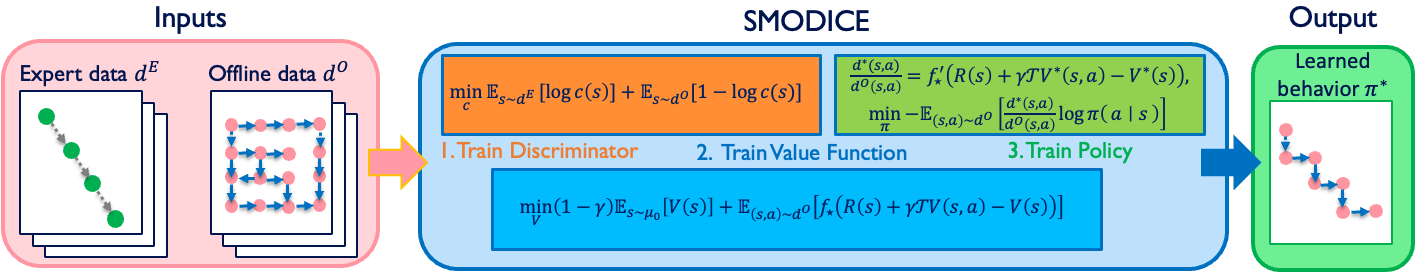}
\vspace{-0.3cm}
\caption{Diagram of SMODICE. First, a state-based discriminator is trained using the offline dataset $d^O$ and expert observations (resp. examples) $d^E$. Then, the discriminator is used to train the Lagrangian value function. Finally, the value function provides the importance weights for policy training, which outputs the learned policy $d^*$.}
\vspace{-0.5cm}
\label{figure:smodice-concept-figure}
\end{figure*}

For these reasons, many algorithms for \textit{online} IL have already sought to remove these assumptions~\cite{torabi2018behavioral, torabi2019generative, liu2019state, radosavovic2020stateonly, eysenbach2021replacing}, but extending them to offline IL remains an open problem. 

We propose \textbf{S}tate \textbf{M}atching \textbf{O}ffline \textbf{DI}stribution \textbf{C}orrection \textbf{E}stimation (SMODICE), a general offline IL framework that can be applied to all three problem settings described above. At a high level, SMODICE is based on a state-occupancy matching view of IL:
\begin{equation}
    \label{eq:state-f-divergence-intro}
    \min_\pi \mathrm{D}_\KL(d^\pi(s) \| d^E(s)),
\end{equation}
which aims to minimize the KL-divergence of the state-occupancy $d$ between the imitator $\pi$ and the expert $E$. This state-occupancy matching objective intuitively demands inferring the correct actions from the offline data in order to match the state-occupancy of the provided expert demonstrations. 
This formulation naturally enables imitation when expert actions are unavailable, and even when the expert's embodiment or dynamics are different, as long as there is a shared task-relevant state.
Finally, we show that example-based RL~\cite{eysenbach2021replacing}, where only examples of successful states are provided as supervision, can be formulated as a state-occupancy matching problem between the imitator and a ``teleporting'' expert that is able to reach success states in one step. Hence, SMODICE can also be used as an offline example-based RL\footnote{We refer to this problem as ``offline imitation learning from examples'' to unify nomenclature with the other two problems.} method without any modification.

Despite its generality, naively optimizing the state-occupancy matching objective would result in an actor-critic style IL algorithm akin to prior work~\cite{ho2016generative, kostrikov2018discriminator, Kostrikov2020Imitation}; however, these algorithms suffer from training instability in the offline regime~\cite{kumar2019stabilizing, lee2021optidice, kim2022demodice} due to the entangled nature of actor and critic learning, leading to erroneous value bootstrapping~\cite{levine2020offline}. SMODICE bypasses this issue by first introducing a $f$-divergence \textit{regularized} state-matching objective and then using its dual optimal solution to formulate a weighted regression policy objective that amounts to behavior cloning of the optimal policy.
Specifically, leveraging the notion of Fenchel conjugacy~\cite{Rockafellar+2015, nachum2020reinforcement}, SMODICE reduces the dual problem of the proposed regularized state-occupancy matching problem to an unconstrained convex optimization problem over a value function (Step 2 in Figure~\ref{figure:smodice-concept-figure}). This unconstrained problem admits closed-form solutions in the tabular case and can be easily optimized using stochastic gradient descent (SGD) in the deep RL setting.
 Then, without any additional learning step, applying Fenchel duality to the optimal value function directly obtains the optimal \textit{primal} solution, which recovers the optimal importance weights for weighted regression (Step 3 in Figure~\ref{figure:smodice-concept-figure}). Note that SMODICE does not optimize this policy objective until the value function has converged; despite forgoing direct minimization of the state-matching objective, this uninterleaved optimization is favorable in the offline setting due to its much improved training stability.

Through extensive experiments, we show that SMODICE is effective for all three problem settings we consider and outperforms all state-of-art methods in each respective setting. We obtain all SMODICE results using a \textit{single} set of hyperparameters, modulo a choice of $f$-divergence which can be tuned \textit{offline}. In contrast, prior methods suffer from much greater performance fluctuation across tasks and settings, validating the stated stability improvement of our optimization approach. Altogether, our proposed method SMODICE can serve as a versatile offline IL algorithm that is suitable for a wide range of assumptions on expert data. 

In summary, our contributions are: (i) SMODICE: a simple, stable, and versatile state-occupancy matching based offline IL algorithm for both tabular and high-dimensional continuous MDPs, (ii) a reduction of example-based reinforcement learning to state-occupancy matcjomg, and (iii) extensive experimental analysis of SMODICE in offline imitation from observations, mismatched experts, and examples; in all three, SMODICE outperforms competing methods.

\begin{figure}[t!]
\centering
\label{figure:illustrative_example}
\subfigure[Mismatched experts]{\label{fig:tabular_diagonal}\includegraphics[width=0.45\columnwidth]{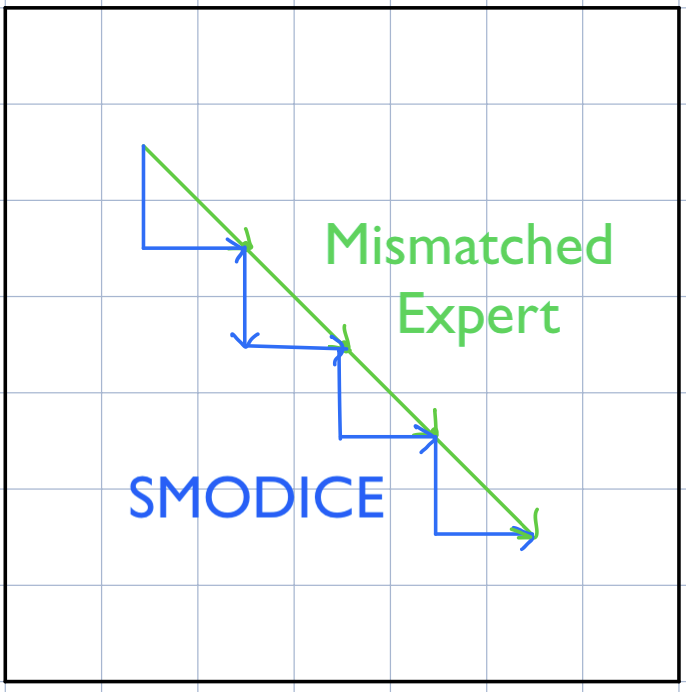}}
\subfigure[Offline IL from examples]{\label{fig:tabular_example}\includegraphics[width=0.455\columnwidth]{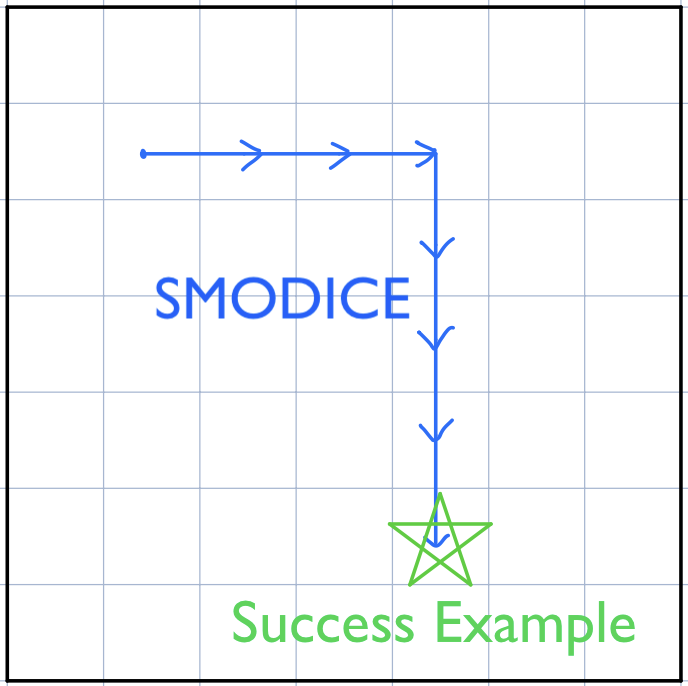}}
\vspace{-0.3cm}
\caption{Illustrations of tabular SMODICE for offline imitation learning from mismatched experts and examples.}
\vspace{-0.5cm}
\end{figure}

\para{Pedagogical examples.} To illustrate SMODICE's versatility, we have applied it to two gridworld tasks, testing offline IL from mismatched experts and examples, respectively. Figure~\ref{fig:tabular_diagonal} shows an expert agent that can move diagonally in any direction, whereas the imitator can only move horizontally or vertically. In Figure~\ref{fig:tabular_example}, only a success state (the star) is provided as supervision. An offline dataset collected by a random agent is given to SMODICE for training in both cases. As shown, SMODICE recovers an optimal policy (i.e. minimum state-occupancy divergence to that of the expert) in both cases. See Appendix~\ref{appendix:tabular-gridworld-experiments} for details. 

\section{Preliminaries}
\para{Markov decision processes.}
We consider a time-discounted Markov decision process (MDP)~\cite{puterman2014markov} $\mc{M}=(S,A,R,T, \mu_0, \gamma)$ with state space $S$, action space $A$, deterministic rewards $R(s,a)$, stochastic transitions $s' \sim T(s,a)$, initial state distribution $\mu_0(s)$, and discount factor $\gamma \in (0, 1]$. A policy $\pi:S \rightarrow \Delta(A)$ determines the action distribution conditioned on the state. 

The state-action occupancies (also known as stationary distribution) $d^\pi(s,a): \mc{S} \times \mc{A} \rightarrow [0,1]$ of $\pi$ is
\begin{equation}
\label{eq:pi-occupancies}
\begin{split}
d^\pi(s,a) \coloneqq \;&(1-\gamma) \sum_{t=0}^{\infty} \gamma^t \text{Pr}(s_t=s, a_t=a \mid  \\ 
&s_0 \sim \mu_0, a_t \sim \pi(s_t), s_{t+1} \sim T(s_t,a_t))
\end{split}
\end{equation}
which captures the relative frequency of state-action visitations for a policy $\pi$. The state occupancies then marginalize over actions: $d^\pi(s) = \sum_a d^\pi(s,a)$. 
The state-action occupancies satisfy the single-step transpose Bellman equation:
\begin{equation}
\label{eq:transpose-bellman-equation}
d^\pi(s,a) = (1-\gamma) \mu_0(s)\pi(a\mid s) + \gamma \cdot \mathcal{T}^\pi_{\star}d^\pi(s,a),
\end{equation}
where $\mathcal{T}^\pi_{\star}$ is the adjoint policy transition operator,
\begin{equation}
\label{eq:adjoint-transition-operator}
\mathcal{T}^\pi_{\star}d^\pi(s,a) \coloneqq \pi(a\mid s) \sum_{\tilde{s}, \tilde{a}}T(s\mid \tilde{s}, \tilde{a})d(\tilde{s}, \tilde{a})
\end{equation}
\para{Divergences and Fenchel conjugates.}
Next, we briefly introduce $f$-divergence and their Fenchel conjugates. 
\begin{definition}[$f$-divergence]
\rm
Given a continuous, convex function $f$ and two probability distributions $p,q \in \Delta(\mc{X})$ over a domain $\mc{X}$, the $f$-divergence of $p$ at $q$ is 
\begin{equation}
\label{eq:f-divergence}
\D_f(p \| q) = \BE_{x \sim  q}\left[f\left(\frac{p(x)}{q(x)}\right)\right]
\end{equation}
\end{definition}

A common $f$-divergence in machine learning is the KL-divergence, which corresponds to $f(x) = x\log x$. Now, we introduce Fenchel conjugate for $f$-divergences.
\begin{definition}[Fenchel conjugate]
\rm
Given a vector space $\Omega$ with inner-product $\langle \cdot, \cdot \rangle$, the \textit{Fenchel conjugate} $f_\star: \Omega_\star \rightarrow \mathbb{R}$ of a convex and differentiable function $f: \Omega \rightarrow \mathbb{R}$ is
\begin{equation}
\label{eq:fenchel-conjugate}
f_\star(y) \coloneqq \max_{x\in \Omega} \langle x, y \rangle - f(x)
\end{equation}
and any maximizer $x^*$ of $f_\star(y)$ satisfies $x^* = f’_\star(y)$.

For an $f$-divergence, under mild realizability assumptions~\cite{dai2016learning} on $f$, the Fenchel conjugate of $D_f(p\|q)$ at $y: \mc{X} \rightarrow \mathbb{R}$ is 
\begin{align}
\label{eq:fenchel-conjugate-f-divergence}
\D_{\star, f}(y) &= \max_{p \in \Delta(\mc{X})} \BE_{x \sim p}[y(x)] - \D_f(p\|q) \\ 
&=\BE_{x \sim q}[f_\star(y(x))]
\end{align}
and any maximizer $p^*$ of $\D_{\star, f}(y)$ satisfies 
\begin{equation}
\label{eq:kkt-conditions}
p^*(x) = q(x)f'_\star(y(x)).
\end{equation} 
This result can be seen as an application of the KKT conditions to problems involving $f$-divergence regularization.
\end{definition} 

\para{Offline imitation learning.}
Many imitation learning approaches rely on minimizing the $f$-divergence between the state-action occupancies of the imitator and the expert~\cite{ho2016generative, ke2020imitation, ghasemipour2019divergence}: 
\begin{equation}
\label{eq:f-divergence-objective}
\min_\pi \D_f\left(d^\pi(s,a) \| d^E(s,a)\right)
\end{equation}
In imitation learning, we do not have $d^E$; instead, we are provided with expert demonstrations $\mc{D}^E \coloneqq \{(s^{(i)}, a^{(i)})\}_{i=1}^N$. 

In offline imitation learning, the agent further cannot interact with the MDP $\mc{M}$; instead, they are given a static dataset of logged transitions $\mc{D}^O \coloneqq \{\tau_i\}_{i=1}^M$, where each trajectory $\tau^{(i)} = (s_0^{(i)},a_0^{(i)},s_1^{(i)},...)$ with $s_0^{(i)} \sim \mu_0$; we denote the empirical state-action occupancies of $\mc{D}^O$ as $d^O(s,a)$.

\section{The SMODICE Algorithm} 
\label{sec:smodice}
In this section, we derive the SMODICE algorithm. We begin by introducing our $f$-divergence regularized offline state-matching objective (Section~\ref{section:f-divergence-regularized-state-matching}). Then, we describe the 3 disjoint training steps of SMODICE in order (Section~\ref{section:smodice-discriminator-training}--\ref{section:smodice-weighted-regression}). Finally, we present SMODICE tailored to tabular MDPs (Section~\ref{section:smodice-tabular-mdps}).

\subsection{$f$-Divergence Regularized State-Matching} 
\label{section:f-divergence-regularized-state-matching}
Recall that the state-occupancy matching objective takes the form
\begin{equation}
    \label{eq:state-f-divergence}
    \min_\pi \mathrm{D}_\KL(d^\pi(s) \| d^E(s)),
\end{equation}
which requires on-policy samples from $\pi$, as the expectation is over $d^\pi$. To enable offline optimization, we necessarily need to involve the offline dataset distribution $d^O$ in our objective.



%
%
%
%
%
%

%
%
First, we assume expert coverage of the offline data:
\begin{assumption}
\label{assumption:expert-coverage}
\rm
$d^O(s) > 0$ whenever $d^E(s) > 0$.
\end{assumption}
This assumption ensures that the offline dataset has coverage over the expert state-marginal, and is necessary for imitation learning to succeed. Whereas prior offline RL approaches~\cite{kumar2020conservative, ma2021conservative} assume full coverage of the state-action space, our assumption\footnote{Furthermore, it is not needed in practice, and is only required for our technical development to ensure that all state-occupancy quantities are well-defined (i.e., no division-by-zero).}  is considerably weaker since it only requires expert coverage.  Given this assumption, we introduce our $f$-divergence regularized state-matching objective, which follows from an upper bound on state-occupancy matching that incorporates the offline dataset distribution $d^O$:
\begin{theorem}
\label{theorem:state-matching-upper-bound}
Given Assumption \ref{assumption:expert-coverage}, we have
\begin{equation}
\label{eq:state-matching-upper-bound}
\begin{split}
&\D_\KL(d^\pi(s)\|d^E(s)) \leq\\
&\BE_{s \sim d^\pi}\left[\log\left(\frac{d^O(s)}{d^E(s)}\right)\right] + \D_\KL(d^\pi(s,a) \| d^O(s,a))
\end{split}
\end{equation}
Furthermore, for any $f$-divergence such that $\D_f \geq \D_\KL$,
\begin{equation}
\label{eq:state-matching-upper-bound-f}
\begin{split}
&\D_\KL(d^\pi(s)\|d^E(s)) \leq\\
&\BE_{s \sim d^\pi}\left[\log\left(\frac{d^O(s)}{d^E(s)}\right)\right] + \D_f(d^\pi(s,a) \| d^O(s,a))
\end{split}
\end{equation}
\end{theorem}
We refer to the RHS of Equation~\eqref{eq:state-matching-upper-bound-f} as the \textit{$f$-divergence regularized state-occupancy matching objective.} The proofs of this theorem and all other theoretical results are in Appendix~\ref{appendix:proofs}. Intuitively, the upper bound states that that offline state-occupancy matching can be achieved by matching states in the offline data that resemble expert states (the first term) with reward function $R(s) = \log \frac{d^E(s)}{d^O(s)}$ (we describe how to compute this reward below), while remaining in the support of the offline state-action distribution (the second term). Replacing KL-divergence with other $f$-divergences can be useful since the conjugate of KL divergence involves a log-sum-exp, which has been found to be numerically unstable in many RL tasks~\cite{zhu2020off, lee2021optidice, rudner2021on}. Now, we describe the three disjoint steps of SMODICE as presented in Figure~\ref{figure:smodice-concept-figure}.

\subsection{Discriminator training}
\label{section:smodice-discriminator-training}
First, we discuss how to compute $R(s)= \log \frac{d^E(s)}{d^O(s)}$. In the tabular case, $R(s)$ can be computed using empirical estimates of $d^E(s)$ and $d^O(s)$. In the continuous case, we can train a discriminator $c:\mc{S} \rightarrow (0,1)$:
\begin{equation}
    \label{eq:adversarial-training}
    \min_c \BE_{s \sim d^E}\left[\log c(s) \right] + \BE_{s \sim d^O}\left[\log 1-c(s) \right]
\end{equation}
The optimal discriminator is $c^\star(s) = \frac{d^O(s)}{d^E(s)+d^O(s)}$~\cite{goodfellow2014generative}, so we can use $R(s) = -\log \left(\frac{1}{c^\star(s)}-1\right)$. 

\subsection{Dual Value Function Training}
\label{section:smodice-dual-value-function-training}
Note that \eqref{eq:state-matching-upper-bound-f} requires samples from $d^\pi$, so it still cannot be easily optimized without online interaction. To address this, we first rewrite it as an optimization problem over the space of valid state-action occupancies~\cite{puterman2014markov}:
\begin{align}
\label{eq:smodice-problem}
(\mathrm{P})\quad &\max_{d(s,a) \geq 0} \BE_{s \sim d(s,a)}\left[R(s)\right] - \D_f(d \| d^O) \\ 
\label{eq:smodice-bellman-flow-constraint}
&\text{s.t. } \sum_a d(s,a) = (1-\gamma)\mu_0(s) + \gamma \mathcal{T}_\star d(s), \forall s \in S 
\end{align}
where $\mathcal{T}_\star d(s) = \sum_{\bar{s}, \bar{a}}T(s \mid \bar{s}, \bar{a})d(\bar{s}, \bar{a})$; here, \eqref{eq:smodice-bellman-flow-constraint} ensures that $d$ is the occupancy distribution for some policy. We assume that \eqref{eq:smodice-problem} is \textit{strictly feasible}. 
\begin{assumption}
\label{assumption:strict-feasibility}
\rm
There exists at least one $d(s,a)$ such that constraints \eqref{eq:smodice-bellman-flow-constraint} are satisfied and $\forall s \in \mathcal{S}, d(s) > 0$.
\end{assumption}
This assumption is mild and can be satisfied in practice for any MDP for which every state is reachable from the initial state distribution. Next, we can form the dual of \eqref{eq:smodice-problem}:
\begin{equation}
\label{eq:smodice-dual-lagrangian}
\begin{split}
(\mathrm{D})\quad &\max_{d(s,a)\geq 0} \min_{V(s)\geq 0} \BE_{s \sim d}\left[R(s)\right] - \D_f(d \| d^O)\\
&+ \sum_s V(s)\left((1-\gamma)\mu_0(s) + \gamma \mathcal{T}_\star d(s) - \sum_a d(s,a) \right)
\end{split}
\end{equation}
where $V(s)$ are the Lagrangian multipliers. Now, because $\mathcal{T}_\star$ is the adjoint of $\mathcal{T}$, we have the following:
\begin{equation}
\label{eq:adjoint-transform}
\sum_{s}V(s) \cdot \mathcal{T}_{\star}d(s) = \sum_{s,a} d(s,a) \cdot (\mathcal{T}V)(s,a)
\end{equation}
Using this equation, we can write \eqref{eq:smodice-dual-lagrangian} as
\begin{equation}
\label{eq:smodice-dual-lagrangian-2}
\begin{split}
(\mathrm{D})\quad &\max_{d(s,a)\geq 0} \min_{V(s)\geq 0} (1-\gamma)\BE_{s\sim \mu_0}[V(s)]\\
&+ \BE_{(s,a) \sim d}\left[R(s) + \gamma \mathcal{T}V(s,a) - V(s)\right]\\
&- \D_f(d(s,a) \| d^O(s,a)) \\ 
\end{split}
\end{equation}
We note that the original problem \eqref{eq:smodice-problem} is convex~\cite{lee2021optidice}. By Assumption \ref{assumption:strict-feasibility}, it is strictly feasible, so by strong duality, we can change the order of optimization in \eqref{eq:smodice-dual-lagrangian-2}: 
\begin{equation}
    \label{eq:smodice-dual-lagrangian-final}
\begin{split}
    (\mathrm{D})\quad & \min_{V(s)\geq0} \max_{d(s,a)\geq0 } (1-\gamma)\BE_{s\sim \mu_0}[V(s)]\\ 
    &+ \BE_{(s,a) \sim d}\left[\left(R(s) + \gamma \mathcal{T}V(s,a) - V(s)\right)\right]\\ 
    &- \D_f(d(s,a) \| d^O(s,a))
\end{split}
\end{equation}

Finally, using the Fenchel conjugate, \eqref{eq:smodice-dual-lagrangian-final} can be reduced to a single unconstrained optimization problem over $V:\mc{S} \rightarrow \mathbb{R}_{\geq 0}$ that depends on samples from only $d^O$ and not $d$; we also obtain the importance weight of the state-occupancy of the optimal policy with respect to the offline data.
\begin{theorem}
\label{theorem:smodice-dual-lagrangian-fenchel}
The optimization problem \eqref{eq:smodice-dual-lagrangian-final} is equivalent to 
\begin{equation}
   \label{eq:smodice-dual-lagrangian-fenchel}
\begin{split}
(\mathrm{D})\quad &\min_{V(s)\geq0} (1-\gamma)\BE_{s\sim \mu_0}[V(s)]\\
&+ \BE_{(s,a) \sim d^O}\left[f_\star\left(R(s) + \gamma \mathcal{T}V(s,a) - V(s)\right)\right]
\end{split}
\end{equation}
Furthermore, given the optimal solution $V^*$, the optimal state-occupancy importance weights are
\begin{equation}
    \label{eq:smodice-fenchel-solutions}
    \frac{d^*(s,a)}{d^O(s,a)} = f'_\star(R(s)+\gamma \mathcal{T}V^*(s,a)-V^*(s))
\end{equation}
\end{theorem}
This result can be viewed as using Fenchel duality to generalize prior DICE-based offline approaches~\cite{lee2021optidice, kim2022demodice}. In particular,
the inner maximization problem in \eqref{eq:smodice-dual-lagrangian-final} is precisely the Fenchel conjugate of $D_f(d(s,a) \| d^O(s,a))$ at $R(s) + \gamma \mathcal{T}V(s,a) - V(s)$ (compare \eqref{eq:smodice-dual-lagrangian-final} to \eqref{eq:fenchel-conjugate-f-divergence}). Similarly, \eqref{eq:smodice-fenchel-solutions} can be derived from leveraging the relationship between the optimal solutions of a pair of Fenchel primal-dual problems (Equation~\eqref{eq:kkt-conditions}). This generality allows us to choose problem-specific $f$-divergences that improve stability during optimization. In Appendix \ref{appendix:smodice-example}, we specialize the SMODICE objective for the KL- and $\chi^2$-divergences, which we use in our experiments.

\subsection{Weighted-Regression Policy Training}
\label{section:smodice-weighted-regression}
Finally, using the optimal importance weights, we can extract the optimal policy $\pi$ using weighted Behavior Cloning:
\begin{equation}
    \label{eq:weighted-bc-objective}
    \begin{split}
    &\min_\pi -\BE_{(s,a)\sim d^*}[\log \pi(a\mid s)]\\
    =& \min_\pi -\BE_{(s,a) \sim d^O}[\xi^*(s,a) \log \pi(a\mid s)]
    \end{split}
\end{equation}
where $\xi^*(s,a) = \frac{d^*(s,a)}{d^O(s,a)}$. Here, $V(s)$ can be viewed as the value function---it is trained by minimizing a convex function of the Bellman residuals and the values of the initial states. Then, it can be used to inform policy learning.

Putting everything together, SMODICE can achieve stable policy learning through a sequence of three \textit{disjoint} supervised learning problems, summarized in Algorithm \ref{alg:smodice-deep-abbreviated}. The full pseudo-code is in Algorithm~\ref{alg:smodice-deep} in Appendix~\ref{alg:smodice-deep}.
\begin{algorithm}[t!]
\caption{SMODICE}\label{alg:smodice-deep-abbreviated}
\begin{algorithmic}[1]
\STATE \textcolor{purple}{\texttt{// Discriminator Learning}}
\STATE Train discriminator $c^*(s)$ using \eqref{eq:adversarial-training} and derive $R(s)$.
\STATE \textcolor{purple}{\texttt{// Value Learning}}
\STATE Train derived value function $V(s)$ using \eqref{eq:smodice-dual-lagrangian-fenchel} 
\STATE \textcolor{purple}{\texttt{// Policy Learning}}
\STATE Derive optimal ratios $\xi^*(s,a)$ through \eqref{eq:smodice-fenchel-solutions} 
\STATE Train policy $\pi$ using weighted BC \eqref{eq:weighted-bc-objective}
\end{algorithmic}
\end{algorithm}

\subsection{SMODICE for Tabular MDPs.}
\label{section:smodice-tabular-mdps}
An appealing property of SMODICE is that it admits closed-form analytic solution in the tabular case. The proof is given in Appendix~\ref{appendix:tabular-smodice}.
\begin{theorem}
\label{theorem:smodice-tabular-closed-form}
Let $R(s) = \log \frac{d^E(s)}{d^O(s)} \in \mathbb{R}_+^{|\mc{S}|}$, and define $\mathcal{T} \in \mathbb{R}^{|\mc{S}||\mc{A}| \times |\mc{S}|}$ and $\mc{B} \in \mathbb{R}^{|\mc{S}||\mc{A}| \times |\mc{S}|}$ by $(\mathcal{T} V)(s,a)  = \sum_{s'} T(s'|s,a) V(s')$ and $(\mathcal{B}V)(s,a) =  V(s)$. Additionally, denote $\mu_0 \in \Delta(|\mc{S}|)$ and $D = \mathrm{diag}(d^O) \in \mathbb{R}^{|\mc{S}||\mc{A}| \times |\mc{S}||\mc{A}|}$. Then, choosing the $\chi^2$-divergence in \eqref{eq:smodice-dual-lagrangian-fenchel}, we have
\begin{equation}
\label{eq:smodice-tabular-optimal-V}
\begin{split}
V^* = &\left((\gamma \mc{T} - \mc{B})^\top D (\gamma \mc{T} - \mc{B})\right)^{-1}\\
&\left((\gamma-1)\mu_0 + (\mc{B} - \gamma \mc{T})^\top D(I+BR) \right)
\end{split}
\end{equation}
\end{theorem}
In Appendix~\ref{appendix:tabular-smodice}, we also derive a finite-sample performance guarantee of SMODICE in the tabular setting.

\section{Offline Imitation Learning from Examples} 

Next, we describe how SMODICE can be applied to offline imitation learning from examples. Starting from the original problem objective from~\citet{eysenbach2021replacing}, we derive a state-occupancy matching objective, enabling us to apply SMODICE without any modification.

\para{Problem setting.} We assume given success examples $S^* = \{s^* \sim p_U(s_t\mid e_t=1)\}$, where $e \in \{0,1\}$  indicates whether the current state is a success outcome, and offline data $\mc{D} = \{(s,a,s')\}$. Here, $U$ is the state distribution of the ``user'' providing success examples. Then, ~\citet{eysenbach2021replacing} proposes the example-based RL objective
\begin{equation}
\label{eq:example-based-rl}
\arg \max_\pi \log p^\pi(e_{t+}=1) = \log \mathbb{E}_{s \sim \mu_0}\left[ p^\pi(e_{t+} = 1 | s_0) \right]
\end{equation}
That is, we want a policy that maximizes the probability of reaching success states in the future. To tackle this problem in the offline setting, our strategy is to convert \eqref{eq:example-based-rl} into an optimization problem over the state-occupancy space.

\para{Intuition.} By parameterizing the problem in terms of state occupancies, a policy that reaches success states in the future is one that has non-zero occupancies at these states---i.e., $d^\pi(s)$ corresponds to a policy that reaches success states if $d^\pi(s)>0$ for $s \in \mc{S}^*$. Furthermore, treating success states as absorbing states in the MDP, then $\sum_{s \in \mc{S}^*} d^\pi(s)$ should ideally be \textit{much} larger than $\sum_{s \notin \mc{S}^*} d^\pi(s)$ (we validate this on gridworld; see Appendix~\ref{appendix:tabular-gridworld-experiments}).

\para{Derivation.} We first transform the problem into state-occupancy space---i.e.,
\begin{equation}
\label{eq:example-rl-occupancy-problem}
\max_\pi \log \mathbb{E}_{s \sim \mu_0}\left[p^\pi(e_{t+} = 1 | s_0) \right] = \max_{d \geq 0} \log \mathbb{E}_{s\sim d(s)}\left[p(e | s)\right]
\end{equation}
which is valid given that the original objective can be thought of as a regular RL problem with reward function $r(s) = p(e\mid s)$~\cite{eysenbach2021replacing}.

Given this formulation, we can derive a tractable lower bound to \eqref{eq:example-rl-occupancy-problem} through Jensen's inequality and Bayes' rule:
\begin{align*} 
\begin{split}
&\log \mathbb{E}_{s\sim d(s)}\left[p_U(e \mid s)\right]\\ 
\ge& \mathbb{E}_{s\sim d(s)}\left[\log p_U(e \mid s)\right] \\
=& \mathbb{E}_{s\sim d(s)}\left[ \log \frac{p_U(s \mid e) p_U(e)}{p_U(s)}\right] \\ 
=&  \mathbb{E}_{s\sim d(s)}\left[ \log \frac{p_U(s \mid e) }{d(s)} \right] +
\mathbb{E}_{s\sim d(s)}\left[\log \frac{d(s)}{p_U(s)}\right] +\mathrm{const.} \\
=& -\D_\KL\left(d(s) \| p_U(s\mid e)\right) + \D_\KL\left(d(s) \| p_U(s) \right) +\mathrm{const.} \\
\ge& -\D_\KL\left(d(s) \| p_U(s\mid e)\right) +\mathrm{const.} \\
\end{split}
\end{align*}
We can optimize the original objective by maximizing this lower bound. Doing so is equivalent to solving 
\begin{equation}
\label{eq:smodice-for-example}
    \min_{d\geq 0} \D_\KL\left(d(s) \| p_U(s\mid e)\right), 
\end{equation}
which is exactly in the form of the state-occupancy matching objective \eqref{eq:state-f-divergence} in the scope of SMODICE. Furthermore, this objective admits an intuitive explanation from a purely imitation learning lens. We can think of $p_U(s \mid e)$ as the state-occupancy distribution of an expert agent who can ``teleport'' to any success state in one time-step. Therefore, we have shown that example-based RL can be understood as a state-occupancy minimization problem between a MDP-dynamics abiding imitator and a teleporting expert agent. Consequently, SMODICE can be used in the offline setting without any algorithmic modification.

\section{Related Work}

\para{Offline imitation learning.} 
The closest work is concurrent work, DEMODICE~\cite{kim2022demodice}, a state-action based offline IL method, also using the DICE paradigm to estimates the occupancy ratio between the expert and the imitator; we overview the DICE literature in Appendix~\ref{appendix:related-work}. Due to its dependence on expert actions, DEMODICE cannot be applied to the three problem settings we study. At a technical level, a key limitation of DEMODICE is that it does not exploit the form of general Fenchel duality and only support the KL-divergence, forgoing other $f$-divergences that can lead to more stable optimization~\cite{ghasemipour2019divergence, ke2020imitation, zhu2020off}. Another related work is ORIL~\cite{zolna2020offline}, which adapts GAIL~\cite{ho2016generative} to the offline setting.
Finally, there has been recent work learning a pessimistic dynamics model using the offline dataset and then performs imitation learning by minimizing the state-action occupancy divergence with respect to the expert inside this learned model~\cite{chang2021mitigating}. As with DEMODICE, this approach requires expert actions and cannot be applied to the settings we study.

\begin{figure}[t!]
\subfigure[Mujoco]{\includegraphics[width=0.33\columnwidth]{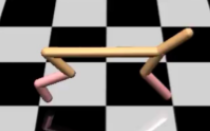}}
\subfigure[AntMaze]{\label{fig:antmaze}\includegraphics[width=0.3\columnwidth]{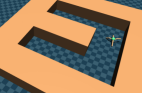}}
\subfigure[Franka Kitchen]{\label{fig:kitchen}\includegraphics[width=0.3\columnwidth]{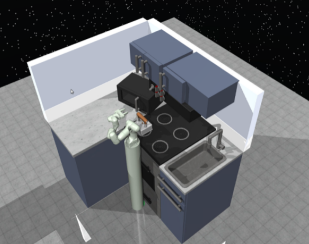}}
\vspace{-0.2cm}
\caption{\textbf{Illustrations of the evaluation environments.}}
\vspace{-0.6cm}
\label{figure:env-illustration}
\end{figure}

\para{Imitation from observations, imitation with mismatched experts, and example-based RL} 
All three of these problems have been studied in the online setting. IfO is often achieved through training an additional inverse dynamics model to infer the expert actions~\cite{torabi2018behavioral, torabi2019generative, liu2019state, radosavovic2020stateonly, gangwani2020state}; in contrast, SMODICE matches the expert observations by identifying the correct actions supported in the offline data. To handle experts with dynamics mismatch, some work explicitly learns a correspondence between the expert and the imitator MDPs~\cite{kim2020domain, raychaudhuri2021crossdomain}; however, these approaches make much stronger assumptions on access to the expert MDP that are difficult to satisfy in the offline setting, such as demonstrations from auxillary tasks. In contrast, SMODICE falls under the category of state-only imitation learning~\cite{liu2019state, radosavovic2020stateonly}, which overcomes expert dynamics differences by only matching the shared task-relevant state space (e.g., $xy$ coordinates for locomotion tasks). Finally, example-based RL was first studied in~\citet{eysenbach2021replacing}; they introduce a recursive-classifier based off-policy actor critic method to solve it. By casting this problem as state-occupancy matching between an imitator and a ``teleporting'' expert agent, SMODICE can solve the offline variant of this problem without modification. 

\section{Experiments}

We experimentally demonstrate that SMODICE is effective for offline IL from observations, mismatched experts, and examples. We give additional experimental details in Appendices \ref{appendix:oilo}, \ref{appendix:oilhe}, and \ref{appendix:oile}, and videos on the project website\footnote{Code is available at:  \href{https://github.com/JasonMa2016/SMODICE}{https://github.com/JasonMa2016/SMODICE}}. 

\begin{figure*}[t!]
\includegraphics[width=\textwidth]{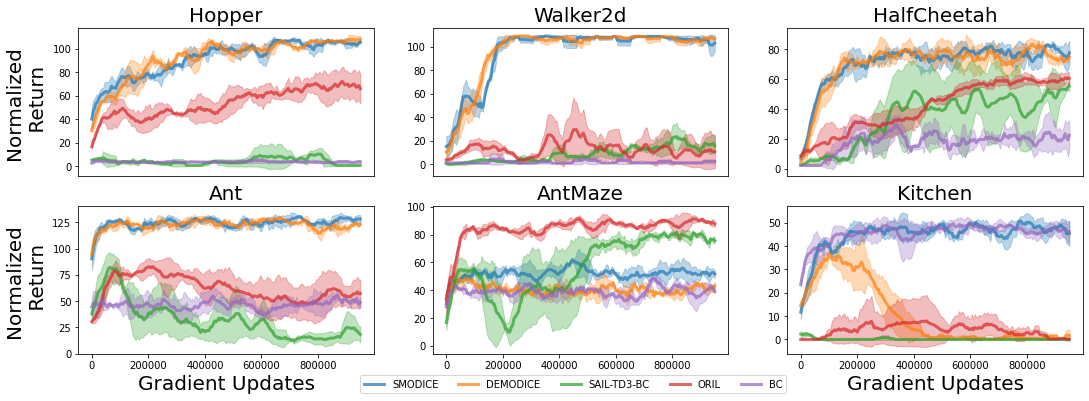}
\vspace{-0.7cm}
\caption{\textbf{Offline imitation learning from observations results.}}
\vspace{-0.5cm}
\label{figure:offline-il-observations}
\end{figure*}

\subsection{Offline Imitation Learning from Observations} 
\label{section:offline-il-observations}

\para{Datasets.} We utilize the D4RL~\cite{fu2021d4rl} offline RL dataset. The dataset compositions for all tasks are listed in Table \ref{table:d4rl-datasets} in Appendix \ref{appendix:oilo}. We consider the following standard Mujoco environments: Hopper, Walker2d, HalfCheetah, and Ant. For each, we take a single expert trajectory from the respective ``expert-v2'' dataset as the expert dataset and omit the actions. For the offline dataset, following~\citet{kim2022demodice}, we use a mixture of small number of expert trajectories ($\leq 200$ trajectories) and a large number of low-quality trajectories from the ``random-v2'' dataset (we use the full random dataset, consisting of around 1 million transitions). This dataset composition is particularly challenging as the learning algorithm must be able to successfully distinguish expert from low-quality data in the offline dataset.

We also include two more challenging environments from D4RL: AntMaze and Franka Kitchen. In AntMaze (Figure \ref{fig:antmaze}), an Ant agent is tasked with navigating an U-shaped maze from one end to the other end (i.e., the goal region). The offline dataset (i.e., ``antmaze-umaze-v2'') consists of trajectories ($\approx$ 300k transitions) of an Ant agent navigating to the goal region from initial states;  The trajectories are not always successful; often, the Ant flips over to its legs before it reaches the goal. We visualize this dataset on the project website. As above, we additionally include 1 million random-action transitions to increase the task difficulty. We take one trajectory from the offline dataset that successfully reaches the goal to be the expert trajectory. Franka Kitchen (Figure \ref{fig:kitchen}), introduced by~\citet{gupta2019relay}, involves controlling a 9-DoF Franka robot to manipulate common household kitchen objects (e.g., microwave, kettle, cabinet) sequentially to achieve a pre-specified configuration of objects. The dataset (i.e., ``kitchen-mixed-v0'') consists of \textit{undirected} human teleoperated demonstrations, meaning that each trajectory only solves a subset of the tasks. Together, these six tasks (illustrated in Figure \ref{figure:env-illustration}) require scalability to high-dimensional state-action spaces and robustness to different dataset compositions.

\para{Method and baselines.} We use SMODICE with $\chi^2$-divergence for all tasks (in other problem settings as well) except Hopper, Walker, and Halfcheetah, where we find SMODICE with KL-divergence to perform better; in Appendix \ref{appendix:smodice-choosing-f-divergence}, we explain how to choose the appropriate $f$-divergence \textit{offline} by monitoring SMODICE's policy loss. For comparisons, we consider both IfO and regular offline IL methods, which make use of expert actions. For the former, we compare against (i) \textbf{SAIL-TD3-BC}, which combines a state-of-art state-matching based \textit{online} IL algorithm (SAIL)~\cite{liu2019state} with a state-of-art offline RL algorithm (TD3-BC)~\cite{fujimoto2021minimalist},\footnote{We chose TD3-BC due to its simplicity and stability.}
(ii) \textbf{Offline Reinforced Imitation Learning (ORIL)}~\cite{zolna2020offline}, which adapts GAIL~\cite{ho2016generative} to the offline setting by using an offline RL algorithm for policy optimization; we implement ORIL using the same state-based
discriminator as in SMODICE, and TD3-BC as the offline RL algorithm. For the latter, we consider the state-of-art \textbf{DEMODICE}~\cite{kim2022demodice} as well as \textbf{Behavior Cloning (BC)}. We train all algorithms for 1 million gradient steps and keep track of the normalized score (i.e., 100 is expert performance, 0 is random-action performance) during training; the normalized score is averaged over 10 independent rollouts. All methods are evaluated over 3 seeds, and one standard-deviation confidence intervals are shaded.

\para{Results.} As shown in Figure~\ref{figure:offline-il-observations}, only SMODICE achieves stable and good performance in all six tasks. It achieves (near) expert performance in all the Mujoco environments, performing on-par with DEMODICE and doing so without the privileged information of expert actions. SMODICE's advantage over DEMODICE is more apparent in AntMaze and Kitchen. In the former, SMODICE outperforms BC, while DEMODICE cannot; in the latter, DEMODICE quickly collapses due to its use of KL-divergence, which may be numerically unstable in high-dimensional environments. Furthermore, we adapt DEMODICE to the state-only setting by training a state-based discriminator; in Appendix~\ref{appendix:oilo-additional-results}, we report the results and find DEMODICE to significantly underperform in the most challenging tasks across three settings.

BC is a strong baseline for tasks where the offline dataset contains (near) expert data (i.e., AntMaze and Kitchen); however, as the dataset becomes more diverse, BC's performance drops significantly. SAIL-TD3-BC and ORIL both fail to learn in some environments and otherwise converge to a worse policy than SMODICE. The only exception is AntMaze; however, in Appendix~\ref{appendix:oilo-additional-results}, we show that both methods collapse with a more diverse version of the AntMaze offline dataset, indicating that unlike SMODICE, these methods are highly sensitive to the composition of the offline dataset, and work best with task-aligned offline data. The sub-par performances of SAIL and ORIL highlight the challenges of adapting online IL methods to the offline setting; we hypothesize that it is not sufficient to simply equip the original methods (i.e., SAIL and GAIL) with a strong base offline RL algorithm. Together, these results demonstrate that SMODICE is stable, scalable, and robust, and significantly outperforms prior methods. Finally, in Appendix~\ref{appendix:oilo-additional-results}, we ablate SMODICE by zeroing out its discriminator-based reward to validate that SMODICE's empirical performance comes from its ability to discriminate expert data in the offline dataset.

\begin{figure*}[t!]
\includegraphics[width=\textwidth]{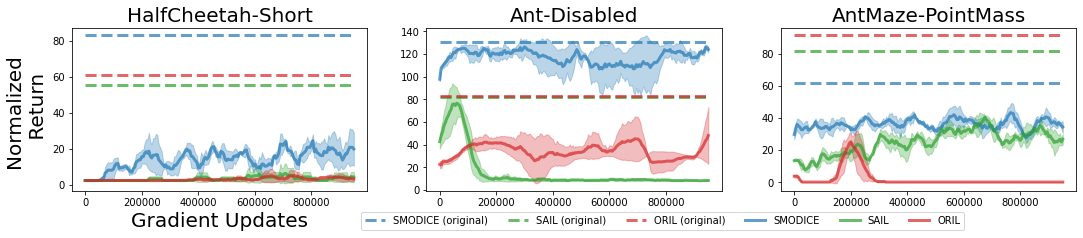}
\vspace{-0.7cm}
\caption{\textbf{Offline imitation learning from mismatched experts results.}}
\vspace{-0.3cm}
\label{figure:offline-il-observations-mismatched}
\end{figure*}

\begin{figure*}[t!]
\includegraphics[width=\textwidth]{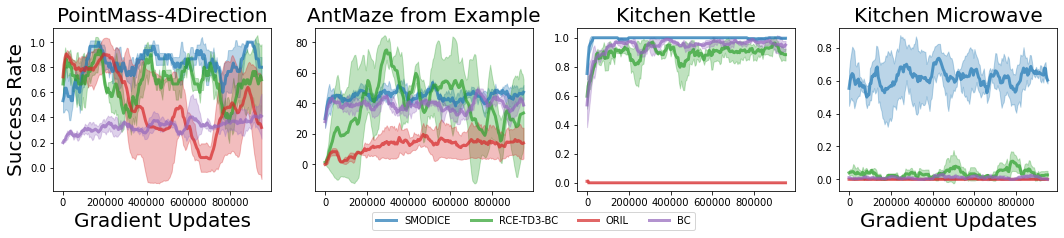}
\vspace{-0.7cm}
\caption{\textbf{Offline imitation learning from examples results.}}
\vspace{-0.4cm}

\label{figure:offline-il-example}
\end{figure*}

\subsection{Offline IL from Mismatched Experts} 
\label{section:offline-il-mismatched}

\para{Datasets and baselines.} We compare SMODICE to SAIL-TD3-BC and ORIL, which are both state-based offline IL methods; in particular, we note that SAIL is originally designed to be robust to mismatched experts.  We consider only tasks in which both SAIL-TD3-BC and ORIL obtained non-trivial performance, including HalfCheetah, Ant, and AntMaze. Then, for each environment, we train a mismatched expert and collect one expert trajectory, replacing the original expert trajectory used in Section~\ref{section:offline-il-observations}. The mismatched experts for the respective tasks are (i) ``HalfCheetah-Short'', where the torso of the cheetah agent is halved in length, (ii) ``Ant-Disabled'', where the front legs are shrank by a quarter in length, and (iii) a 2D PointMass agent operating in the same maze configuration. The mismatched experts are illustrated in Figure~\ref{figure:mismatched-expert} in Appendix~\ref{appendix:oilhe} and the project website. For the first two, we train an expert policy using SAC~\cite{haarnoja2018soft} and collect one expert trajectory. The latter task is already in D4RL; thus, we take one trajectory from ``maze2d-umaze-v0'' as the expert trajectory. Because Ant and PointMass have different state spaces, following~\citet{liu2019state}, we train the discriminator on the shared $xy$-coordinates of the two state spaces. The offline datasets are identical to the ones in Section~\ref{section:offline-il-observations}.

\para{Results.} The training curves are shown in Figure \ref{figure:offline-il-observations-mismatched}; we illustrate the original maximum performance attained by each method (i.e., using the original expert trajectory, Section~\ref{section:offline-il-observations}) using dashed lines as points of reference. As can be seen, SMODICE is significantly more robust to mismatched experts than either SAIL-TD3-BC or ORIL. On AntMaze, the task where SAIL-TD3-BC and ORIL originally outperform SMODICE, learning from a PointMass expert significantly deteriorates their performances, and the learned policies are noticably worse than that of SMODICE, which has the smallest performance drop. The other two tasks exhibit similar trends; SMODICE is able to learn an expert level policy for the original Ant embodiment using a disabled Ant expert, and is the only method that shows any progress on the hardest HalfCheetah-Short task. Despite using the same discriminator for reward supervision, SMODICE is substantially more robust than ORIL, likely due to the occupancy-constraint $\D_f(d(s,a) \|d^O(s,a))$ term in its objective \eqref{eq:state-matching-upper-bound-f}, which ensures that the learned policy is supported by the offline data as it attempts to match the expert states. On the project website, we visualize SMODICE and ORIL policies on all tasks. In Appendix \ref{appendix:quantitative-analysis-oil-mismatched}, we provide additional quantitative analysis of Figure \ref{figure:offline-il-observations-mismatched}.

\subsection{Offline Imitation Learning from Examples}
\label{section:offline-il-examples}

\para{Tasks.} We use the AntMaze and Kitchen environments and create example-based task variants. For AntMaze, we replace the full demonstration with a small set of success states (i.e., Ant in the goal region) extracted from the offline data. For Kitchen, we consider two subtasks in the environment: Kettle and Microwave. and define task success to be only whether the specified object is correctly placed (instead of all objects as in the original task); the success states are extracted from the offline data accordingly. Examples of the success states are illustrated in Figure \ref{figure:success-states} in Appendix \ref{appendix:oile}.  Note that the kitchen dataset contains many trajectories where the kettle is moved first. Thus, the kettle task is easy even for Behavior Cloning (BC), since cloning the offline data can lead to success. This is not the case for the microwave task, making it much more difficult to solve using only success examples. In addition, we introduce the PointMass-4Direction environment. Here, a 2D PointMass agent is tasked with navigating to the middle point of a specified edge of the square that encloses the agent (see Figure \ref{figure:success-states}(a)). The offline dataset is generated using a waypoint navigator controlling the agent to each of the four possible goals and contains equally many trajectories for each goal; we visualize this dataset on the project website. At training and evaluation time, we set the left edge to be the desired edge and collect success states from the offline data accordingly. This task is low-dimensional but consists of multi-task offline data, making it challenging for algorithms such as BC that do not solve the example-based RL objective. 

\setlength\intextsep{0pt}
\begin{wrapfigure}{r}{0.5\linewidth}
\includegraphics[width=0.5\columnwidth]{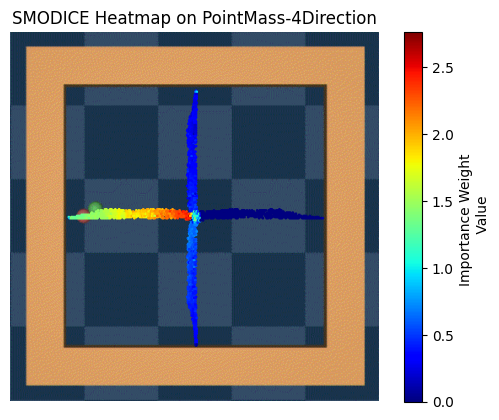}
\vspace{-0.9cm}
\caption{\textbf{SMODICE weights.}}
\label{figure:smodice-pointmass-importance-weight}
\end{wrapfigure}

\para{Approaches.} We make no modification to SMODICE; the only difference is that the discriminator is trained using success states instead of full expert state trajectories. Our main comparison is \textbf{RCE-TD3-BC}, which combines RCE~\cite{eysenbach2021replacing}, the state-of-art online example-based RL method, and TD3-BC. We also compare against ORIL~\cite{zolna2020offline}, using the same architecture as in Section~\ref{section:offline-il-observations}. Finally, we also include BC.

\para{Results.} As shown in Figure \ref{figure:offline-il-example}, SMODICE is the best performing method on all four tasks and is the only one that can solve the Microwave task; we visualize all methods' policies on all tasks on the project website. RCE-TD3-BC is able to solve the first three tasks, but achieves worse solutions and exhibits substantial performance fluctuation during training; we posit that the optimization for RCE, which requires alternate updates to a recursive classifier and a policy, is substantially more difficult than that of SMODICE. ORIL is unstable and fails to make progress in most tasks. Interestingly, as in the mismatched expert setting, on AntMaze, ORIL's performance is far below that of SMODICE, despite attaining better results originally (Figure~\ref{figure:offline-il-observations}). This comparison demonstrates the versatility of SMODICE afforded by its state-occupancy matching objective; in contrast, ORIL treats offline IL from examples as an offline RL task with discriminator-based reward and cannot solve the task.

To better understand SMODICE, on PointMass-4Direction, we visualize the importance weights $\xi(s,a)$ it assigns to the offline dataset. As shown in Figure~\ref{figure:smodice-pointmass-importance-weight}, SMODICE assigns much higher weights to transitions along the correct path from the initial state region to the success examples. Interestingly, the weights progressively decrease along this path, indicating that SMODICE has learned that it must pay more attention transitions at the beginning of the path, since making a mistake there is more likely to derail progress towards the goal. This behavior occurs \textit{automatically} via SMODICE's state-matching objective without any additional bias.
\vspace{-0.3cm}
\section{Conclusion}
We have proposed SMODICE, a simple, stable, and versatile algorithm for offline imitation learning from observations, mismatched experts, and examples. Leveraging Fenchel duality, SMODICE derives the optimal dual value function to the state-occupancy matching objective, and obtains an uninterleaved optimization procedure for its value and policy networks that is favorable in the offline setting. Through extensive experiments, we have shown that SMODICE significantly outperforms prior state-of-art methods in all three settings. We believe that the generality of SMODICE's optimization procedure invites many future work directions, including offline model-based RL~\cite{yu2020mopo, kidambi2020morel}, safe RL~\cite{ma2021conservative2}, and extending it to visual domains.

\section*{Acknowlegement}
We thank members of Perception, Action, and Learning group at UPenn for their feedback. This work is funded in part by an Amazon Research Award,
gift funding from NEC Laboratories America, NSF Award
CCF-1910769, NSF Award CCF-1917852 and ARO Award
W911NF-20-1-0080. The U.S. Government is authorized to
reproduce and distribute reprints for Government purposes
notwithstanding any copyright notation herein.
\newpage 
\bibliography{references}
\bibliographystyle{icml2022}

\newpage 
\appendix
\onecolumn

\section{Proofs} 
\label{appendix:proofs}
\subsection{Technical Lemmas}
\begin{lemma}
\label{lemma:state-matching-upper-bound}
We have
$$\D_\KL(d^\pi(s)\|d^E(s)) \leq \D_\KL(d^\pi(s,a)\|d^E(s,a))$$
\end{lemma}
\begin{proof}
We first state and prove a related lemma, which first appeared in~\cite{yang2019imitation}.

\begin{lemma}
$$\D_\KL \left(d^\pi(s,a,s') \| d^E(s,a,s')\right) = \D_\KL\left(d^\pi(s,a) \| d^E(s,a) \right).$$
\end{lemma}
\begin{proof}
\begin{align*}
    &\D_\KL \left(d^\pi(s,a,s') \| d^E(s,a,s')\right)\\ 
    =& \int_{\mc{S} \times \mc{A} \times \mc{S}} d^\pi(s,a,s') \log \frac{d^\pi(s,a) \cdot T(s' \mid s,a)}{d^E(s,a) \cdot T(s'\mid s,a)} ds' dads \\ 
    =& \int_{\mc{S} \times \mc{A} \times \mc{S}} d^\pi(s,a,s') \log \frac{d^\pi(s,a)}{d^E(s,a)} ds' dads \\ 
    =& \int_{\mc{S} \times \mc{A}} d^\pi(s,a)\log \frac{d^\pi(s,a)}{d^E(s,a)} dads \\ 
    =&\D_\KL\left(d^\pi(s,a) \| d^E(s,a) \right)
\end{align*}
\end{proof}
Using this result, we can show the desired upper bound:
\begin{align*}
    &\D_\KL\left(d^\pi(s,a) \| d^E(s,a) \right) \\
    =& \D_\KL \left(d^\pi(s,a,s') \| d^E(s,a,s')\right) \\
    =& \int_{\mc{S} \times \mc{A} \times \mc{S}} d^\pi(s,a,s') \log \frac{d^\pi(s,a) \cdot T(s' \mid s,a)}{d^E(s,a) \cdot T(s'\mid s,a)} ds' dads \\ 
    =& \int_{\mc{S} \times \mc{A} \times \mc{S}} d^\pi(s) \pi(a\mid s) T(s'\mid s,a) \log \frac{d^\pi(s,a) \cdot T(s' \mid s,a)}{d^E(s,a) \cdot T(s'\mid s,a)} ds' dads \\ 
    =& \int d^\pi(s) \pi(a\mid s) T(s'\mid s,a) \log \frac{d^\pi(s)}{d^E(s)} ds'dads + \int d^\pi(s) \pi(a\mid s) T(s'\mid s,a) \log \frac{\pi(a\mid s) T(s' \mid s,a)}{\pi^E(a\mid s) T(s'\mid s,a)} ds'dads \\
    =& \int d^\pi(s) \log \frac{d^\pi(s)}{d^E(s)} ds + \int d^\pi(s) \pi(a\mid s) \log \frac{\pi(a\mid s)}{\pi^E(a\mid s)} da ds  \\
    =& \D_\KL\left(d^\pi(s) \| d^E(s)\right) + \D_\KL\left(\pi(a \mid s) \| \pi^E(a \mid s)\right) \\
    \geq& \D_\KL\left(d^\pi(s) \| d^E(s) \right) 
\end{align*}
\end{proof} 

\subsection{Proof of Theorem \ref{theorem:state-matching-upper-bound}} 
\begin{proof}
\begin{align*}
    &D_\KL\left(d^\pi(s) \| d^E(s) \right) \\
    =& \int d^\pi(s) \log \frac{d^\pi(s)}{d^E(s)}\cdot \frac{d^O(s)}{d^O(s)} ds, \quad \text{we assume that $d^O(s) > 0$ whenever $d^E(s) > 0$.} \\ 
    =& \int d^\pi(s) \log \frac{d^O(s)}{d^E(s)} ds + \int d^\pi(s) \log \frac{d^\pi(s)}{d^O(s)} ds \\
    \leq& \BE_{s \sim d^\pi} \left[\log \frac{d^O(s)}{d^E(s)} \right] + \D_\KL\left(d^\pi(s,a) \| d^E(s,a) \right)
\end{align*}
where the last step follows from Lemma \ref{lemma:state-matching-upper-bound}. Then, for any $\D_f \geq \D_\KL$, we have that
$$
D_\KL\left(d^\pi(s) \| d^E(s) \right)  \leq \BE_{s \sim d^\pi} \left[\log \frac{d^O(s)}{d^E(s)} \right] + \D_f\left(d^\pi(s,a) \| d^E(s,a) \right)
$$
\end{proof}

\subsection{Proof of Theorem \ref{theorem:smodice-dual-lagrangian-fenchel}}
\begin{proof}
We begin with 
\begin{equation}
    \min_{V(s)\geq0} \max_{d(s,a)\geq0 } (1-\gamma)\BE_{s\sim \mu_0}[V(s)]
    + \BE_{(s,a) \sim d}\left[\left(R(s) + \gamma \mathcal{T}V(s,a) - V(s)\right)\right]
    - \D_f(d(s,a) \| d^O(s,a))
\end{equation}
We have that 
\begin{align}
        &\min_{V(s)\geq0} \max_{d(s,a)\geq0 } (1-\gamma)\BE_{s\sim \mu_0}[V(s)]
    + \BE_{(s,a) \sim d}\left[\left(R(s) + \gamma \mathcal{T}V(s,a) - V(s)\right)\right]
    - \D_f(d(s,a) \| d^O(s,a)) \\ 
    =&     \min_{V(s)\geq0} (1-\gamma)\BE_{s\sim \mu_0}[V(s)] + \max_{d(s,a)\geq0 } 
    + \BE_{(s,a) \sim d}\left[\left(R(s) + \gamma \mathcal{T}V(s,a) - V(s)\right)\right]
    - \D_f(d(s,a) \| d^O(s,a)) \\ 
\label{eq:V-fenchel-dual}
    =&\min_{V(s)\geq0} (1-\gamma)\BE_{s\sim \mu_0}[V(s)] + \BE_{(s,a) \sim d^O}\left[f_\star\left(R(s) + \gamma \mathcal{T}V(s,a) - V(s)\right)\right]
\end{align}
where the last step follows from recognizing that the inner-maximization is precisely the Fenchel conjugate of $D_f(d(s,a) \| d^O(s,a))$ at $R(s) + \gamma \mathcal{T}V(s,a) - V(s)$. 

To show the relationship among $V^\star$ and $\xi^\star$, we recognize that  \eqref{eq:V-fenchel-dual} and \eqref{eq:smodice-problem} are a pair of Fenchel primal-dual problems.
\begin{lemma}
\label{lemma:fenchel-primal-dual}
    $$ \min_{V(s)\geq0} (1-\gamma)\BE_{s\sim \mu_0}[V(s)] + \BE_{(s,a) \sim d^O}\left[f_\star\left(R(s) + \gamma \mathcal{T}V(s,a) - V(s)\right)\right]$$ is the Fenchel dual to 
    \begin{align}
    \label{eq:d-fenchel-primal}
            &\max_{d(s,a) \geq 0} \BE_{s \sim d}\left[\log\left(\frac{d^E(s)}{d^O(s)}\right)\right] - \D_f(d(s,a) \| d^O(s,a)) \\ 
    \label{eq:bellman-flow-constraint}
    &\text{s.t. } \sum_a d(s,a) = (1-\gamma)\mu_0(s) + \gamma \mathcal{T}_\star d(s), \forall s \in S 
    \end{align}
\end{lemma}
\begin{proof}
We define the indicator function $\delta_{\mc{X}}(x)$ as 
$$
\delta_{\mc{X}}(x) = \begin{cases} 0 & x \in \mc{X} \\ 
\infty & \text{otherwise}
\end{cases} 
$$
Then, we define $g:\mathbb{R}^{|\mc{S}|} \rightarrow \mathbb{R}$ as $g(\cdot) \coloneqq \delta_{\{(1-\gamma)\mu_0\}}(\cdot)$. Then, it can be shown that the Fenchel conjugate of $g$ is $g_\star(\cdot) = (1-\gamma) \BE_{\mu_0}[\cdot]$. In addition, we denote $h(\cdot) \coloneqq \D+f(\cdot \| d^O)$; then, $h_\star(\cdot) = \BE_{(s,a) \sim d^O}[f_\star(\cdot)]$. Finally, define matrix operator $A \coloneqq \gamma \mc{T}_\star - I$. Using these notations, we can write \eqref{eq:V-fenchel-dual} as 
\begin{equation}
\label{eq:V-fenchel-dual-linear-operator-format}
    \min_V g_\star(V) + h_\star(A_\star V + R)
\end{equation}
Then, we proceed to derive the Fenchel dual of \eqref{eq:V-fenchel-dual-linear-operator-format}:
\begin{align}
    & \min_V g_\star(V) + h_\star(A_\star V + R) \\ 
    \label{lemma:3-step1}
    =& \min_V \max_d g_\star(V) + \langle d, A_\star V + R \rangle - h(d) \\ 
    =&  \min_V \max_d g_\star(V) + \langle d, A_\star V \rangle  + \langle d, R \rangle - h(d) \\ 
    \label{lemma:3-step3}
    =& \max_d \left(\min_V g_\star(V) + \langle d, A_\star V \rangle \right) + \langle d, R \rangle - h(d) \\ 
    \label{lemma:3-step4}
    =& \max_d \left(\min_V g_\star(V) + \langle Ad, V \rangle \right) + \langle d, R \rangle - h(d) \\ 
    =& \max_d \left(\max_V -g_\star(V) + \langle -Ad, V \rangle \right) + \langle d, R \rangle - h(d) \\ 
    \label{lemma:3-steplast}
    =& \max_d g(-Ad) + \langle d, R \rangle - h(d) 
\end{align}
where \eqref{lemma:3-step1} follows applying Fenchel conjugacy to $h_\star$, \eqref{lemma:3-step3} follows from strong duality, \eqref{lemma:3-step4} follows from the property of an adjoint operator, and \eqref{lemma:3-steplast} follows from applying Fenchel conjugacy to $g_\star$. Here, we recognize that \eqref{lemma:3-steplast} is precisely the optimization problem \eqref{eq:d-fenchel-primal}-\eqref{eq:bellman-flow-constraint}, where we have moved the constraint \eqref{eq:bellman-flow-constraint} to the objective as the indicator function $g(-Ad)$:
\begin{align*} &g(-Ad) = \delta_{\{(1-\gamma)\mu_0\}} \left(d - \gamma \mc{T}_\star d \right)\\ 
\Leftrightarrow& \sum_a d(s,a) = (1-\gamma)\mu_0(s) + \gamma \mathcal{T}_\star d(s), \forall s \in S 
\end{align*}
\end{proof}
Giving Lemma \ref{lemma:fenchel-primal-dual}, we use the fact that $d^*$ and $V^*$ admit the following relationship:
\begin{equation}
    d^* = h'_\star(-A_\star V^*+R)
\end{equation}
This follows from the characterization of the optimal solutions for a pair of Fenchel primal-dual problems with convex $g,h$ and linear operator $A$~\cite{nachum2020reinforcement}. In this case, assuming that we can exchange the order of expectation and derivative (e.g, conditions of Dominated Convergence Theorem hold), we have
\begin{equation}
    d^* = \BE_{(s,a)\sim d^O}\left[f_\star \left((R(s) + \gamma \mc{T}V(s,a)-V(s)\right)  \right],
\end{equation}
or equivalently,
\begin{equation}
    d^*(s,a) = f_\star \left(R(s) + \gamma \mc{T}V(s,a)-V(s) \right) \cdot d^O (s,a), \forall s, a \in \mc{S} \times \mc{A}, 
\end{equation}
as desired. 
\end{proof}

\section{Extended Related Work}
\label{appendix:related-work}
\para{Stationary distribution correction estimation.}
Estimating the optimal policy's stationary distribution using off-policy data was introduced by~\cite{nachum2019dualdice} as the DICE trick. This technique has been shown to be effective for off-policy evaluation~\cite{nachum2019dualdice, Zhang*2020GenDICE:, dai2020coindice}, policy optimization~\cite{nachum2019algaedice, lee2021optidice}, online imitation learning~\cite{Kostrikov2020Imitation, zhu2020off}, and concurrently, offline imitation learning~\cite{kim2022demodice}. Within the subset of DICE-based policy optimization methods, none has tackled state-occupancy matching or directly apply Fenchel Duality to its full generality to arrive at the form of value function objective we derive.

\section{SMODICE with common $f$-divergences}
\label{appendix:smodice-example} 

\begin{example}[SMODICE with $\chi^2$-divergence]
Suppose $f(x) = \frac{1}{2}(x-1)^2$, corresponding to $\chi^2$-divergence. Then, we can show that $f_\star(x) = \frac{1}{2}(x+1)^2$ and $f'_\star(x) = x+1$. Hence, the SMODICE objective amounts to
\begin{equation}
\begin{split}
&\min_{V(s)\geq0} (1-\gamma)\BE_{s\sim \mu_0}[V(s)] + \frac{1}{2}\BE_{(s,a) \sim d^O}\left[\left(R(s) + \gamma \mathcal{T}V(s,a) - V(s) + 1\right)^2\right]
\end{split}
\end{equation}
and 
\begin{equation}
     \xi^*(s,a) = \frac{d^*(s,a)}{d^O(s,a)} = \max\left(0, R(s,a) +\gamma \mathcal{T}V^*(s,a)-V^*(s) + 1 \right) 
\end{equation}
\end{example}

\begin{example}[SMODICE with KL-divergence]
We have $f(x) = x\log x$. Using the fact that the conjugate of the negative entropy function, restricted to the probability simplex, is the log-sum-exp function~\cite{boyd2004convex}, it follows that $\D_{\star, f}(y) = \log \BE_{x\sim q}[\mathrm{exp} y(x)]$. Hence, the KL-divergence SMODICE objective is 
\begin{equation}
\begin{split}
&\min_{V(s)\geq0} (1-\gamma)\BE_{s\sim \mu_0}[V(s)]+ \log \BE_{(s,a) \sim d^O}\left[ \mathrm{exp}\left(R(s) + \gamma \mathcal{T}V(s,a) - V(s)\right)\right]
\end{split}
\end{equation}
and 
\begin{equation}
     \xi^*(s,a) = \frac{d^*(s,a)}{d^O(s,a)} = \mathrm{softmax}\left(R+\gamma \mathcal{T}V^*(s,a)-V^*(s)\right)
\end{equation}
\end{example}

\section{SMODICE for Tabular MDPs}
\label{appendix:tabular-smodice}
In this section, we derive the closed-form expression of SMODICE for tabular MDPs. For simplicity, we assume that the expert state occupancies are given, $d^E(s) \in \Delta(|\mc{S}|)$. A behavior policy $\pi_b$ is used to collect the offline dataset $\mc{D}^O$. Then, we can construct a surrogate MDP $\hat{\mc{M}}$ using maximum likelihood estimation (i.e., $\hat{T}(s,a,s') = \frac{n(s,a,s')}{n(s,a)}$). Using $\hat{\mc{M}} $, we can extract the empirical estimate of the behavior policy occupancies $d^O  \in \Delta(|\mc{S}| |\mc{A}|)$ using linear programming. Then, we can define the reward vector $R \in \mathbb{R}_+^{|\mc{S}|}$ as $R(s) = \log \frac{d^E(s)}{d^O(s)}$. Using the $\chi^2$-divergence version of SMODICE, we can write down the objective for $V(s) \in \mathbb{R}_+^{|\mc{S}|}$:
\begin{equation}
\label{eq:smodice-chi-tabular-first}
\begin{split}
&\min_{V(s)\geq0} (1-\gamma)\BE_{s\sim \mu_0}[V(s)] + \frac{1}{2}\BE_{(s,a) \sim d^O}\left[\left(R(s) + \gamma \mathcal{T}V(s,a) - V(s) + 1\right)^2\right]
\end{split}
\end{equation}

We rewrite this expression in vector-matrix form to derive the closed-form solution. To this end, we define $\mathcal{T} \in \mathbb{R}^{|\mc{S}||\mc{A}| \times |\mc{S}|}$ and $\mc{B} \in \mathbb{R}^{|\mc{S}||\mc{A}| \times |\mc{S}|}$ such that $(\mathcal{T} V)(s,a)  = \sum_{s'} T(s'|s,a) V(s')$ and $(\mathcal{B}V)(s,a) =  V(s)$. Additionally, we denote $\mu_0 \in \Delta(|\mc{S}|)$ and $D = \mathrm{diag}(d^O) \in \mathbb{R}^{|\mc{S}||\mc{A}| \times |\mc{S}||\mc{A}|}$. Then, we can rewrite \eqref{eq:smodice-chi-tabular-first}:
\begin{equation}
    \label{eq:smodice-chi-tabular-vector}
    \begin{split}
       &\min_{V(s)\geq0} (1-\gamma)\BE_{s\sim \mu_0}[V(s)] + \frac{1}{2}\BE_{(s,a) \sim d^O}\left[\left(R(s) + \gamma \mathcal{T}V(s,a) - V(s) + 1\right)^2\right] \\ 
       \Rightarrow &\min_{V(s)} (1-\gamma) \mu_0^{\top} V +\frac{1}{2}\BE_{(s,a) \sim d^O}\left[\left(\underbrace{\mathcal{B}R(s,a) + \gamma \mathcal{T}V(s,a) - \mathcal{B}V(s,a)}_{r_V(s,a)} + 1\right)^2\right] \\ 
       \Rightarrow &\min_{V(s)} (1-\gamma) \mu_0^{\top} V + \frac{1}{2} (r_V + I)^{\top} D (r_V+ I)
    \end{split}
\end{equation}
where $r_V \in \mathbb{R}^{|\mc{S}||\mc{A}|}$ and $I$ is the all-one vector in $\mathbb{R}^{|\mc{S}||\mc{A}|}$. Denoting $J(V) \coloneqq (1-\gamma) \mu_0^{\top} V + \frac{1}{2} (r_V + I)^{\top} D (r_V+ I)$, it is clear that $J(V)$ is a convex program in $V$. Therefore, we can find its optimal solution by solving the first-order stationary point. We have:
$$
\begin{aligned}
\frac{\partial J(V)}{\partial V} =& \frac{\partial}{\partial V} \left((1-\gamma) \mu_0^{\top} V + \frac{1}{2} (r_V + I)^{\top} D (r_V+ I) \right) \\ 
=& \frac{\partial}{\partial V} \left((1-\gamma) \mu_0^{\top} V + \frac{1}{2} r_V^{\top}D r_V + r_V^{\top}DI + I^{\top}DI \right) \\ 
=& (1-\gamma)\mu_0 + (\gamma \mc{T} - \mc{B})^{\top}D r_V + (\gamma \mc{T} - \mc{B})^{\top} D I  \\
=& (1-\gamma)\mu_0 + (\gamma \mc{T} - \mc{B})^{\top}D (\mc{B}R + (\gamma \mc{T} - \mc{B})V) + (\gamma \mc{T} - \mc{B})^{\top} D I 
\end{aligned}
$$
Then, by setting this expression to zero and solving for $V$ gives the optimal $V^*$:
\begin{equation}
    \label{eq:smodice-chi-tabular-optimal-V}
    \begin{split}
    &(\gamma \mc{T} - \mc{B})^\top D (\gamma \mc{T} - \mc{B}) V = (\gamma-1)\mu_0 + (\mc{B} - \gamma \mc{T})^\top D(I+BR)\\
    \Rightarrow& V^* = \left((\gamma \mc{T} - \mc{B})^\top D (\gamma \mc{T} - \mc{B})\right)^{-1}\left((\gamma-1)\mu_0 + (\mc{B} - \gamma \mc{T})^\top D(I+BR) \right)
    \end{split} 
\end{equation}
and we can recover $\xi^*(s,a) = \frac{d^*(s,a)}{d^O(s,a)}$:
\begin{equation}
    \label{eq:smodice-chi-tabular-optimal-ratio}
    \xi^*(s,a) = \mc{B}R(s,a) + \gamma \mc{T}V^*(s,a) - \mc{B}V^*(s,a) + 1
\end{equation}

Pythonic pseudo-code using NumPy~\cite{harris2020array} is given in Algorithm 
\ref{alg:smodice-chi-tabular}. 

\subsection{Performance Guarantee}
The closed-form solution of $V^*$ assumes knowledge of the true transition $\mathcal{T}$. When the empirical transition function $\hat{\mathcal{T}}$ is estimated from samples (i.e., $\hat{\mathcal{T}}(s'\mid s,a) := \frac{n(s,a,s')}{n}$), we can obtain the following finite-sample performance guarantee: 
\begin{theorem}
\label{theorem:performance-guarantee}
Let $R_{\max} = \max_s \log \frac{d^E(s)}{d^O(s)}$, $D_{\min} = \min_{s,a} d^O(s,a)$, and $\hat{\mathcal{T}}(s'\mid s,a) := \frac{n(s,a,s')}{n}$. Assume that $\norm{(A^\top DA)^{-1}}_\infty \leq \frac{1}{(1-\gamma)^2 D_{\min}}$\footnote{This assumption is similar to the assumption of a lower bound on the minimum eigenvalue of the covariance matrix required to bound estimation error in linear regression (with $A^\top DA$ being analogous to the covariance matrix).}. Then, for any $\delta \in \mathbb{R}_{>0}$, with probability $\ge1-\delta$,  we have
\begin{equation}
\norm{V^* - \hat{V}}_\infty \leq \left(\frac{2(2+R_{\max})(2+\gamma)\gamma}{(1-\gamma)^4 D_{\min}^2} \right)\sqrt{\frac{2S}{n} \log \frac{4SA}{\delta}}
\end{equation}
\end{theorem}

\begin{proof}
We begin by reiterating the expressions for $V^*$ and $\hat{V}$:
\begin{equation}
    \begin{split}
        V^* &= \left((\gamma \mc{T} - \mc{B})^\top D (\gamma \mc{T} - \mc{B})\right)^{-1}
\left((\gamma-1)\mu_0 + (\mc{B} - \gamma \mc{T})^\top D(I+BR) \right) \\ 
        \hat{V} &= \left((\gamma \hat{\mc{T}} - \mc{B})^\top D (\gamma \hat{\mc{T}} - \mc{B})\right)^{-1}
\left((\gamma-1)\mu_0 + (\mc{B} - \gamma \hat{\mc{T}})^\top D(I+BR) \right) \\ 
    \end{split}
\end{equation}
For notational simplicity, we let $A := \gamma \mc{T} - B$ and $\hat{A} := \gamma \hat{\mc{T}} - B$. Then, we have
\begin{align}
    V^* - \hat{V} &= \left(A^\top D A \right)^{-1}
\left((\gamma-1)\mu_0 - A^\top D(I+BR) \right) - \left(\hat{A}^\top D \hat{A} \right)^{-1}
\left((\gamma-1)\mu_0 - \hat{A}^\top D(I+BR) \right) \\ 
&= (A^\top D A)^{-1} (\gamma-1)\mu_0 \\ -& (A^\top D A)^{-1}A^\top D(I+BR) - (\hat{A}^\top D \hat{A})^{-1} (\gamma-1)\mu_0 + (\hat{A}^\top D \hat{A})^{-1}\hat{A}^\top D(I+BR)
\end{align}
Now, we can bound the $\norm{\cdot}_\infty$:
\begin{align}
    \norm{V^* - \hat{V}}_\infty =& \lVert (A^\top D A)^{-1} (\gamma-1)\mu_0 - (A^\top D A)^{-1}A^\top D(I+BR)- (\hat{A}^\top D \hat{A})^{-1} (\gamma-1)\mu_0 \\+ & (\hat{A}^\top D \hat{A})^{-1}\hat{A}^\top D(I+BR) \rVert_\infty \\
    \leq & \norm{(A^\top D A)^{-1} (\gamma-1)\mu_0 -  (\hat{A}^\top D \hat{A})^{-1} (\gamma-1)\mu_0}_\infty \\ + & \norm{(\hat{A}^\top D \hat{A})^{-1}\hat{A}^\top D(I+BR) -(A^\top D A)^{-1}A^\top D(I+BR)}_\infty  \\
    \leq & (1-\gamma) \norm{(A^\top D A)^{-1} -  (\hat{A}^\top D \hat{A})^{-1}}_\infty \\ + & \norm{(\hat{A}^\top D \hat{A})^{-1}\hat{A}^\top D(I+BR) -(A^\top D A)^{-1}A^\top D(I+BR)}_\infty \\ 
    =& (1-\gamma) \norm{(A^\top D A)^{-1} -  (\hat{A}^\top D \hat{A})^{-1}}_\infty \\
    +&\norm{(\hat{A}^\top D \hat{A})^{-1}\hat{A}^\top D(I+BR) - (A^\top D A)^{-1}\hat{A}^\top D(I+BR)}_\infty\\ 
    +& \norm{(A^\top D A)^{-1}\hat{A}^\top D(I+BR)  - (A^\top D A)^{-1}A^\top D(I+BR)}_\infty \\
    \leq & (1-\gamma) \norm{(A^\top D A)^{-1} -  (\hat{A}^\top D \hat{A})^{-1}}_\infty \\
    + & \norm{ (\hat{A}^\top D \hat{A})^{-1} - (A^\top D A)^{-1}}_\infty \norm{\hat{A}^\top D(I+BR)}_\infty \\ 
    + & \norm{(A^\top D A)^{-1}}_\infty \norm{(\hat{A}-A)^\top D(I+BR)}_{\infty}
\end{align}

Since induced norm is sub-multiplicative, we have
\begin{align}
    \norm{\hat{A}^\top D(I+BR)}_\infty &\leq \norm{\hat{A}^\top D}_\infty \norm{(I+BR)}_\infty \leq (1+R_{\max}) \\ 
    \norm{(\hat{A}-A)^\top D(I+BR)}_{\infty} & \leq \norm{(\hat{A}-A)^\top D}_\infty \ \norm{(I+BR)}_\infty \leq \norm{(\hat{A}-A)^\top D}_\infty (1+R_{\max})
\end{align}
The first inequality follows because
\begin{equation}
    \norm{\hat{A}^\top D}_\infty = \max_{s'} \sum_{s,a} \left\lvert (\gamma \hat{\mc{T}}(s'\mid s,a) - \mathbf{1}(s'=s)) D(s,a) \right\rvert \leq \max_{s',s,a} \left\lvert \gamma \hat{\mc{T}}(s'\mid s,a) - \mathbf{1}(s'=s) \right\rvert = 1
\end{equation}
which uses the fact that $\sum_{s,a}D(s,a) = 1$. 

Plugging this back in gives
\begin{equation}
\label{eq:intermediate-step}
\begin{split}
    \norm{V^* - \hat{V}}_\infty &\leq \left((1-\gamma) + (1+R_{\max})\right) \norm{ (\hat{A}^\top D \hat{A})^{-1} - (A^\top D A)^{-1}}_\infty \\ 
    &+ (1+R_{\max})\norm{(A^\top D A)^{-1}}_\infty \norm{(\hat{A}-A)^\top D}_{\infty}
\end{split}
\end{equation}

Now, we note that
\begin{align}
    &\norm{(\hat{A}-A)^\top D}_{\infty} \\ 
    =& \gamma \norm{(\hat{\mc{T}}-\mc{T})^\top D}_\infty \\ 
    =& \gamma \max_{s'} \sum_{s,a}\left\lvert (\hat{\mc{T}}(s' \mid s,a) - \mc{T}(s'\mid s,a)) D(s,a) \right\rvert \\ 
    \leq & \gamma \max_{s',s,a}\left\lvert (\hat{\mc{T}}(s' \mid s,a) - \mc{T}(s'\mid s,a)) \right\rvert \\ 
    \leq& \gamma \max_{s,a} \norm{\hat{\mc{T}}(\cdot \mid s,a) - \mc{T}(\cdot \mid s,a)}_1
\end{align}
and 
\begin{align}
    &\norm{ (\hat{A}^\top D \hat{A})^{-1} - (A^\top D A)^{-1}}_\infty \\
    =& \norm{(A^\top D A)^{-1}(A^\top D A - \hat{A}^\top D \hat{A})(\hat{A}^\top D \hat{A})^{-1}}_\infty \\ 
    \leq & \norm{(A^\top D A)^{-1}}_\infty \norm{A^\top D A - \hat{A}^\top D \hat{A}}_\infty \norm{(\hat{A}^\top D \hat{A})^{-1}}_\infty \\ 
    \leq & \norm{(A^\top D A)^{-1}}^2_\infty \norm{A^\top D A - \hat{A}^\top D \hat{A}}_\infty \\  
    = & \norm{(A^\top D A)^{-1}}^2_\infty \norm{A^\top D A - A^\top D \hat{A} + A^\top D \hat{A} -  \hat{A}^\top D \hat{A}}_\infty \\ 
    =& \norm{(A^\top D A)^{-1}}^2_\infty \left(\norm{A^\top D (A-\hat{A})}_\infty + \norm{(A-\hat{A})^\top D \hat{A}}_\infty \right) \\ 
    \leq & \norm{(A^\top D A)^{-1}}^2_\infty \left(\norm{A^\top D}_\infty \norm{A-\hat{A}}_\infty + \norm{(A-\hat{A})^\top D}_\infty \norm{\hat{A}}_\infty \right) \\ 
    \leq & \norm{(A^\top D A)^{-1}}^2_\infty \left(\gamma \max_{s,a} \norm{\mc{T}(\cdot \mid s,a)-\hat{\mc{T}}(\cdot \mid s,a)}_1 + (1+\gamma)\gamma \max_{s,a} \norm{\mc{T}(\cdot \mid s,a)-\hat{\mc{T}}(\cdot \mid s,a)}_1 \right) \\ 
    =& \norm{(A^\top D A)^{-1}}^2_\infty \left((2+\gamma)\gamma \max_{s,a}\norm{\mc{T}(\cdot \mid s,a)-\hat{\mc{T}}(\cdot \mid s,a)}_1 \right) 
\end{align} 
where we have used the fact that $\norm{A-\hat{A}}_\infty = \gamma \max_{s,a}\norm{T(\cdot \mid s,a) - \hat{\mc{T}}(\cdot \mid s,a)}_1$ and that
\begin{equation}
    \norm{A}_\infty = \max_{s,a} \sum_{s'} \left\lvert \gamma \mc{T}(s' \mid s,a) - \mathbf{1}(s'=s) \right\rvert \leq  \max_{s,a} \sum_{s' \neq s} \left\lvert \gamma \mc{T}(s' \mid s,a) \right\rvert + 1 \leq 1+\gamma
\end{equation}

Plugging these back into \eqref{eq:intermediate-step} gives 
\begin{equation}
    \begin{split}
        \norm{V^* - \hat{V}}_\infty &\leq \left((1-\gamma) + (1+R_{\max})\right)(2+\gamma)\gamma \norm{(A^\top D A)^{-1}}^2_\infty \max_{s,a}\norm{\mc{T}(\cdot \mid s,a)-\hat{\mc{T}}(\cdot \mid s,a)}_1 \\ 
        & + (1+R_{\max})\gamma \norm{(A^\top D A)^{-1}}_\infty \max_{s,a}\norm{\mc{T}(\cdot \mid s,a)-\hat{\mc{T}}(\cdot \mid s,a)}_1
    \end{split}
\end{equation}
For any $\delta \in [0,1)$, with probability $1-\delta/2$, we have
\begin{equation}
    \max_{s,a}\norm{\mc{T}(\cdot \mid s,a)-\hat{\mc{T}}(\cdot \mid s,a)}_1 \leq \sqrt{\frac{2S}{n} \ln \frac{4SA}{\delta}}
\end{equation}
Then, leveraging our assumption that 
\begin{equation}
\norm{(A^\top DA)^{-1}}_\infty = \frac{1}{\inf_{\norm{x}=1} \norm{(A^\top DA)x}_\infty} \leq \frac{1}{(1-\gamma)^2 D_{\min}}
\end{equation}
we have, with probability $1-\delta$, 
\begin{align}
    \norm{V^* - \hat{V}}_\infty &\leq \left(\frac{(2+R_{\max})(2+\gamma)\gamma}{(1-\gamma)^4 D_{\min}^2} + \frac{(1+R_{\max})\gamma}{(1-\gamma)^2D_{\min}} \right)\sqrt{\frac{2S}{n} \ln \frac{4SA}{\delta}} \\ 
     &\leq \left(\frac{2(2+R_{\max})(2+\gamma)\gamma}{(1-\gamma)^4 D_{\min}^2} \right)\sqrt{\frac{2S}{n} \ln \frac{4SA}{\delta}}
\end{align}
\end{proof}

\subsection{Gridworld Experiments}
\label{appendix:tabular-gridworld-experiments}
In this subsection, we provide more experimental details and analysis of the tabular SMODICE experiments shown in Figure \ref{figure:illustrative_example}.

To generate the offline dataset, a random policy (i.e., a policy that chooses each action with equal probabilities) is executed in the MDP for 10000 epsiodes. We use this dataset to compute the approximate MDP. Then, this MDP is used as an input to SMODICE (see Algorithm \ref{alg:smodice-chi-tabular}). The data collection procedure for the offline imitation learning from examples setting is identical.

\para{Offline IL from mismatched experts.}  In this task, we consider an expert agent that can move one grid cell diagonally in any direction, whereas the imitator is only able to move one grid cell horizontally or vertically. The expert policy is shown in black in Figure~\ref{fig:tabular_diagonal}. Using purely an offline dataset collected by a random agent, we compute the closed-form tabular SMODICE
solution \eqref{eq:smodice-tabular-optimal-V} using Algorithm \ref{alg:smodice-chi-tabular} and obtain the zig-zagging policy shown in blue. Indeed, this solution is one of the two correct solutions that minimize the state-occupancy divergence (the other one mirrors this path along the expert demo), while being feasible under the imitator dynamics.

\para{Offline IL from examples.}
We arbitrarily select a state to be the success state denoted by the green star in Figure~\ref{fig:tabular_example}. In this case, the expert's state occupancies is simply a one-hot vector with weight $1$ at the success state. Then, we again use the tabular version of SMODICE to compute the policy whose state occupancies is as close to this one-hot vector as possible; the solution is illustrated in blue. As can be seen, this policy successfully reaches the goal. Furthermore, it is easy to see that in this task, a policy that minimizes state-occupancy divergence to the expert (i.e., the one-hot vector) is one that reaches the goal with the fewest steps. The policy learned by SMODICE is indeed among the set of optimal policies. 

Furthermore, we compute the state occupancies of all states in the gridworld. For the success state, $d(s) \approx 0.915$, whereas the second largest state occupancy is $0.01$. This validates the intuition that $\sum_{s \in \mc{S}^*} d^\pi(s) \gg \sum_{s \notin \mc{S}^*} d^\pi(s)$.

\begin{algorithm}[hb]
\caption{SMODICE with $\chi^2$-divergence for Tabular MDPs}
\label{alg:smodice-chi-tabular}
\definecolor{codeblue}{rgb}{0.28125,0.46875,0.8125}
\lstset{
    basicstyle=\fontsize{9pt}{9pt}\ttfamily\bfseries,
    commentstyle=\fontsize{9pt}{9pt}\color{purple},
    keywordstyle=
}
\begin{lstlisting}[language=python]
# d_E: the expert state occupancies, |S| 
# mdp: the empirical MDP learned using offline data
# pi_b: the behavior policy, |S||A|

def SMODICE(mdp, d_E, pi_b):
    d_O_sa = compute_policy_occupancies(mdp, pi_b)  # |S||A|
    d_O = d_O_sa.reshape(mdp.S, mdp.A).sum(axis=1) # |S|

    # compute reward function 
    R = np.log(d_E/d_O) # |S|
    
    # define and reshape matrices
    T = mdp.T.reshape(mdp.S * mdp.A, mdp.S)  # |S||A| x |S|
    B = np.repeat(np.eye(mdp.S), mdp.A, axis=0)  # |S||A| x |S|
    I = np.ones(mdp.S * mdp.A) # |S||A|
    D = np.diag(d_O_sa)  # |S||A| x |S||A|
    
    # compute optimal V
    H = (mdp.gamma * P - B).T @ D @ (mdp.gamma * T - B) # |S| x |S|
    y = -((1 - mdp.gamma) * p0 + (mdp.gamma * P - B).T @ D @ (I + B @ R)) # |S|
    V_star = np.linalg.pinv(H) @ y # |S| 

    # compute optimal occupancy ratios
    xi_star = B @ R + (mdp.gamma * P - B) @ V_star + 1 # |S||A|
    m = np.array(xi_star >= 0, dtype=np.float) 
    xi_star = xi_star * m 

    # weighted BC
    pi_star = (xi_star * d_O).reshape(mdp.S, mdp.A) # |S||A|
    pi_star /= np.sum(pi_star, axis=1, keepdims=True)

    f_divergence = d.dot(0.5 * (w_star ** 2))

    return pi_star, f_divergence, V_star
\end{lstlisting}
\end{algorithm}

\section{SMODICE with Deep Neural Networks} 
\label{appendix:deep-smodice}
For high-dimensional MDP with continuous state and action spaces, we instantiate SMODICE using deep neural networks. In particular, we parameterize $V_\theta$ and $\pi_\phi$ using DNNs with weights $\theta$ and $\phi$, respectively. 

\begin{algorithm}[t!]\label{algo:MainAlgorithm}
\caption{SMODICE for Continuous MDPs}\label{alg:smodice-deep}
\begin{algorithmic}[1]
\small
\STATE \textbf{Require}: Expert demonstration(s) $\mc{D}^E$, offline dataset $\mc{D}^O$, choice of $f$-divergence $f$
\STATE Randomly initialize discriminator $c_\psi$, value function $V_\theta$, and policy $\pi_\phi$. 
\STATE \textcolor{purple}{\texttt{// Train Expert (resp. Example) Discriminator}}
\STATE Train $c_\psi$ using $\mc{D}^E$ and $\mc{D}^O$ using Equation \eqref{eq:adversarial-training}

\STATE \textcolor{purple}{\texttt{// Train Lagrangian Value Function}}
\FOR{\text{number of iterations}}
\STATE Sample minibatch of offline data $\{s_t^i, a_t^i, s_{t+1}^i\}_{i=1}^N \sim \mc{D}^O, \{s_0^i\}_{i=1}^M\sim \mc{D}^O(\mu_0)$
\STATE Obtain reward: $R_i = c_\theta(s^i_t), i=1,...,N$
\STATE Compute value objective $\mc{L}(\theta) \coloneqq (1-\gamma) \frac{1}{M}\sum_{i=1}^M V_\theta(s_0^i) + \frac{1}{N} f_\star\left(R^i + \gamma V(s_{t+1}^i)-V(s^i_t)\right)$
\STATE Update $V_\theta$ using SGD: $V_\theta \leftarrow V_\theta - \eta_V \nabla \mc{L}(\theta)$
\ENDFOR

\STATE \textcolor{purple}{\texttt{// Policy Learning}}

\FOR{\text{number of iterations}}
\STATE Sample minibatch of offline data $\{s_t^i, a_t^i, s_{t+1}^i\}_{i=1}^N \sim \mc{D}^O$
\STATE \textcolor{purple}{\texttt{// Compute Optimal Importance Weights}}
\STATE Compute $\xi^*(s^i,a^i) = f'_\star\left(R(s^i) + \gamma V(s_{t+1}^i) - V(s_t^i) \right), i=1,...,N$
\STATE \textcolor{purple}{\texttt{// Weighted Behavior Cloning}}
\STATE Update $\pi_\psi$ using Equation \eqref{eq:weighted-bc-objective}
\ENDFOR
\end{algorithmic}
\end{algorithm}

\para{Remark.} We note that the sample-based estimation of Equation \eqref{eq:smodice-dual-lagrangian-fenchel} (Line 9) is biased because $\mc{T}V$ is itself an expectation that is inside a (non-linear) convex function $f$~\cite{baird1995residual}; however, as in several prior works~\cite{nachum2019algaedice, nachum2020reinforcement, lee2021optidice}, we do not find this biased estimate to impact empirical performance and keep it for simplicity.

\subsection{Hyperparameters and Architecture}
We use the same hyperparameters for all SMODICE experiments in this paper modulo the choice of $f$-divergences (explained in the next section). In terms of architecture, we use a simple 2-layer ReLU network with hidden size 256 to parameterize the value network. For the policy network, we use the same architecture to parameterize a Gaussian output distribution; the mean and the log standard deviation are ouputs of two separate heads. In addition, we use an tanh function on the Gaussian samples to enforce bounded actions, as in~\cite{haarnoja2018soft}. The discriminator uses the same architecture. Table \ref{table:smodice_hyp} summarizes the hyperparameters as well as the architecture.

\begin{table}[ht]
\centering
\caption{SMODICE Hyperparameters.}\label{table:smodice_hyp}
\begin{tabular}{cll}
\toprule
& Hyperparameter & Value \\
\midrule
SMODICE Hyperparameters & Optimizer & Adam~\cite{kingma2014adam} \\
                                      & Critic learning rate & 3e-4 \\
                                      & Discriminator learning rate & 3e-4 \\
                                      & Actor learning rate  & 3e-5 \\
                                      & Mini-batch size      & 256 \\
                                      & Discount factor      & 0.99 \\
                                      & Actor Mean Clipping & (-7.24, 7.24) \\
                                      & Actor Log Std. Clipping & (-5,2) \\
\midrule
Architecture    
                                      & Discriminator hidden dim     & 256        \\
                                      & Discriminator hidden layers  & 2          \\
                                      & Discriminator activation function & Tanh \\
                                      & Critic hidden dim    & 256        \\
                                      & Critic hidden layers & 2          \\
                                      & Critic activation function & ReLU \\
                                      & Actor hidden dim     & 256        \\
                                      & Actor hidden layers  & 2          \\
                                      & Actor activation function & ReLU \\
\bottomrule
\end{tabular}
\end{table} %

\subsection{Choosing $f$-Divergence in Practice}
\label{appendix:smodice-choosing-f-divergence}
In our experiments, SMODICE is implemented using $\chi^2$-divergence for all tasks except Hopper, Walker2d, and HalfCheetah. Here, we show that a suitable choice of $f$-divergence can be chosen \textit{offline} by observing the initial direction of the SMODICE policy loss on the offline dataset. More specifically, on the environments in which SMODICE exhibited largest performance discrepancies between using KL-divergence or $\chi^2$-divergence, we have found that SMODICE returns are \textit{negatively} correlated with the policy loss. As shown in Figure \ref{figure:smodice-choosing-f-divergence}, the poor performing variant of SMODICE always has a policy loss that initially jumps and vice-versa. This makes intuitive sense given the composition of the offline datasets, which is a mix of small amount of expert data with a large amount of poor quality data (see Appendix \ref{appendix:oilo} for more details). When SMODICE fails to pick out the expert data, which is often narrowly distributed, then it must have assigned relatively higher importance weights to the lower quality data, which is more diverse. This creates a more difficult supervised learning task, leading to higher training loss for the policy. Therefore, in practice, we recommend monitoring SMODICE's initial policy loss direction to determine whether the current $f$-divergence will lead to good performance and make changes accordingly. 

\begin{figure}[t!]
\includegraphics[width=\textwidth]{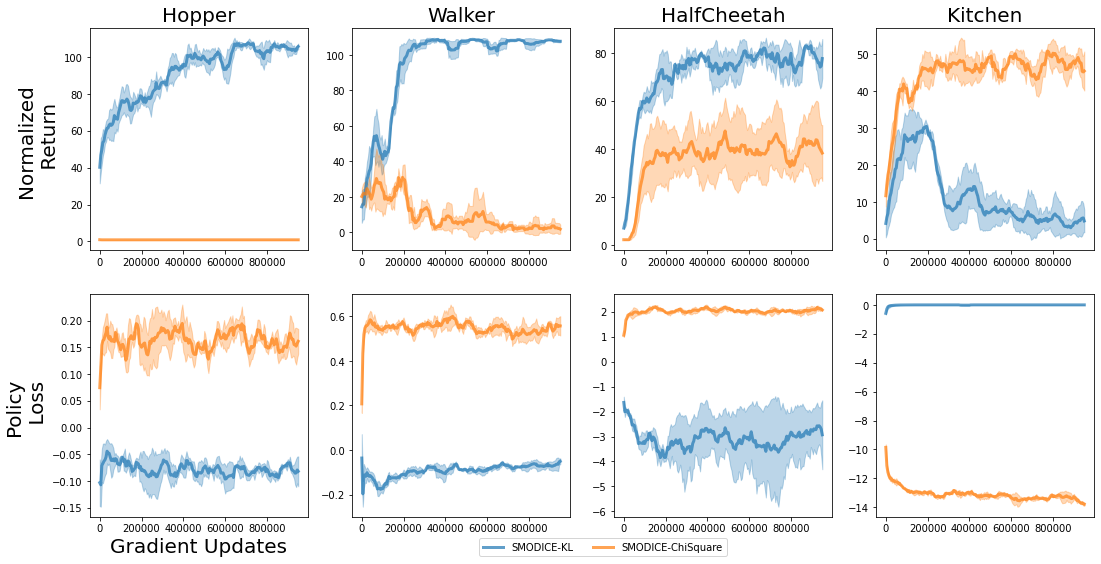}
\caption{SMODICE returns are negatively correlated with the direction of its policy losses.}
\label{figure:smodice-choosing-f-divergence}
\end{figure}

\section{Baselines}

\para{TD3-BC.} Many of our baselines are implemented using TD3-BC as their offline policy optimizer. We use the default hyperparameters for TD3-BC provided by~\citet{fujimoto2021minimalist}, shown in Table \ref{table:td3_hyp}. 

\para{Implementation Details.} We use the official PyTorch implementation of TD3-BC, publicly available at \url{https://github.com/sfujim/TD3_BC}. For DEMODICE, because the code is not public available, we implement it using PyTorch, adapting from \url{https://github.com/secury/optidice}; we use the hyperparameters reported in the paper. Note that DEMODICE shares many architectures with SMODICE. For example, DEMODICE uses a state-action discriminator, and we implement it by simply changing the input space of the state-based discriminator used in our SMODICE implementation. For SAIL, we use the official PyTorch implementation (\url{https://github.com/FangchenLiu/SAIL}) and combine it with TD3-BC. We implement RCE using PyTorch, adapting from the official TensorFlow implementation \url{https://github.com/google-research/google-research/tree/master/rce}.

\begin{table}[ht]
\centering
\caption{TD3+BC Hyperparameters. This table is reproduced from~\citet{fujimoto2021minimalist} directly.}\label{table:td3_hyp}
\begin{tabular}{cll}
\toprule
& Hyperparameter & Value \\
\midrule
TD3 Hyperparameters & Optimizer & Adam~\citep{kingma2014adam} \\
                                      & Critic learning rate & 3e-4 \\
                                      & Actor learning rate  & 3e-4 \\
                                      & Mini-batch size      & 256 \\
                                      & Discount factor      & 0.99 \\
                                      & Target update rate   & 5e-3 \\
                                      & Policy noise         & 0.2 \\
                                      & Policy noise clipping & (-0.5, 0.5) \\
                                      & Policy update frequency & 2 \\
\midrule
Architecture         & Critic hidden dim    & 256        \\
                                      & Critic hidden layers & 2          \\
                                      & Critic activation function & ReLU \\
                                      & Actor hidden dim     & 256        \\
                                      & Actor hidden layers  & 2          \\
                                      & Actor activation function & ReLU \\
\midrule
TD3+BC Hyperparameters  & $\alpha$             & 2.5 \\
\bottomrule
\end{tabular}
\end{table} %

\section{Offline IL from Observations Experimental Details}
\label{appendix:oilo}
\subsection{Datasets}
For Hopper, Walker2d, HalfCheetah, Ant, and AntMaze, we construct the offline datasets by combining a small amount of expert data and a large amount of low quality random data. For the first four tasks, we leverage the respective ``expert-v2'' and ``random-v2'' datasets in the D4RL benchmark. For AntMaze, we use trajectories from ``antmaze-umaze-v2'' as the expert data; for the random data, we simulate the antmaze environment for 1M steps using random actions and take the resulting transitions. For the kitchen environment, we use the full ``kitchen-mixed-v0'' dataset as the offline dataset without further augmentation. See Table \ref{table:d4rl-datasets} for dataset breakdown. 

\begin{table}[]
\centering 
\caption{Offline Dataset Compositions.}
\resizebox{0.7\columnwidth}{!}{
\begin{tabular}{|c|c|c|c|c|}
\hline
Task  & State Dim & Expert Dataset & Expert Data Size  & Random Data Size \\ \hline
Hopper & 11 & hopper-expert-v2 & 193430 & 999999 \\
Walker2d & 17 & walker2d-expert-v2 & 99900 & 999999 \\ 
HalfCheetah & 17 & halfcheetah-expert-v2 & 199800 & 999000 \\
Ant & 27 & ant-expert-v2 & 192409 & 999427 \\ 
AntMaze & 29 & antmaze-umaze-v2 & 349687 & 999000 \\ 
Kitchen & 60 & kitchen-mixed-v0 & 136937 & 0 \\ \hline 
\end{tabular}}
\label{table:d4rl-datasets}
\end{table}

\subsection{Additional Results}
\label{appendix:oilo-additional-results}
In this section, we present some additional results as well as ablation experiments. 

\para{Diverse AntMaze.} In Section \ref{section:offline-il-observations}, we have found that two of the baselines (SAIL-TD3-BC and ORIL) outperform SMODICE on the AntMaze benchmark. To investigate their sources of empirical gain, we have designed a diverse version of the AntMaze dataset to test how different approaches are robust to the dataset composition on the same task. To this end, we take the AntMaze offline dataset (explained above) and reverse half of the trajectories in their directions. In other words, these reversed trajectories would navigate from the original goal to the initial state. This procedure is easy to do because the U-shaped maze is symmetric. Then, using this dataset, we have trained all approaches in Section \ref{section:offline-il-observations} again. As shown in Figure \ref{figure:oil-ablations}(a), on this dataset, both SAIL-TD3-BC and ORIL quickly collapse, indicating that these methods are very brittle to the dataset composition. In contrast, SMODICE remains the best performing algorithm, despite overall drop in all methods' performances. 

\para{SMODICE with Zero Reward.} 
We compare SMODICE with SMODICE-Zero, which simply assigns every transition zero reward (i.e., $R(s) = 0$) regardless of its similarity to an expert state. Then, we compare the ratio of the importance weights (i.e., $\xi(s,a)$) assigned to the offline expert data and the offline random data by the two SMODICE methods, respectively. As shown in Figure \ref{figure:smodice-weight-ratio}, SMODICE assigns much higher relative weights to the expert data and consequently significantly outperforms SMODICE-Zero. These results demonstrate that SMODICE's empirical performance comes from its superior ability to discriminate the offline expert data, which is a by-product of its optimization procedure. 

\para{DEMODICE with State-Based Discriminator.}

We replace DEMODICE's state-action based discriminator with a state-based one to make it compatible with the problem settings we consider in this paper. 
We compare this version of DEMODICE (\textbf{DEMODICE+SD}) to SMODICE in Table~\ref{table:smodice-vs-demodice}, showing performance at convergence. SMODICE significantly outperforms DEMODICE+SD, which suffers from training instability due to optimizing the KL conjugate. Thus, naively adapting DEMODICE to state matching is insufficient; our generalized $f$-divergence based algorithm is crucial for enabling learning from challenging expert observations (e.g., mismatched dynamics or examples).

\begin{table}[h]
\centering 
\vspace{0.1cm}
\caption{SMODICE vs. DEMODICE with State-Discriminator}
\resizebox{0.7\columnwidth}{!}{
\begin{tabular}{|c|c|c|c|c|c|}
\hline
Algorithm  & AntMaze-PointMass &AntMaze-Example & PointMass-4D    & Kettle & Microwave \\ \hline  
DEMODICE+SD  & 19.8 & 32.7 & 0.0 & 0.0 & 0.1    \\ \hline
SMODICE  & \textbf{34.3}& \textbf{47.3} & \textbf{80.0} & \textbf{100.0} & \textbf{60.3} \\ \hline
\end{tabular}}
\vspace{-5pt}
\label{table:smodice-vs-demodice}
\end{table}

\begin{figure}[t!]
\includegraphics[width=\textwidth]{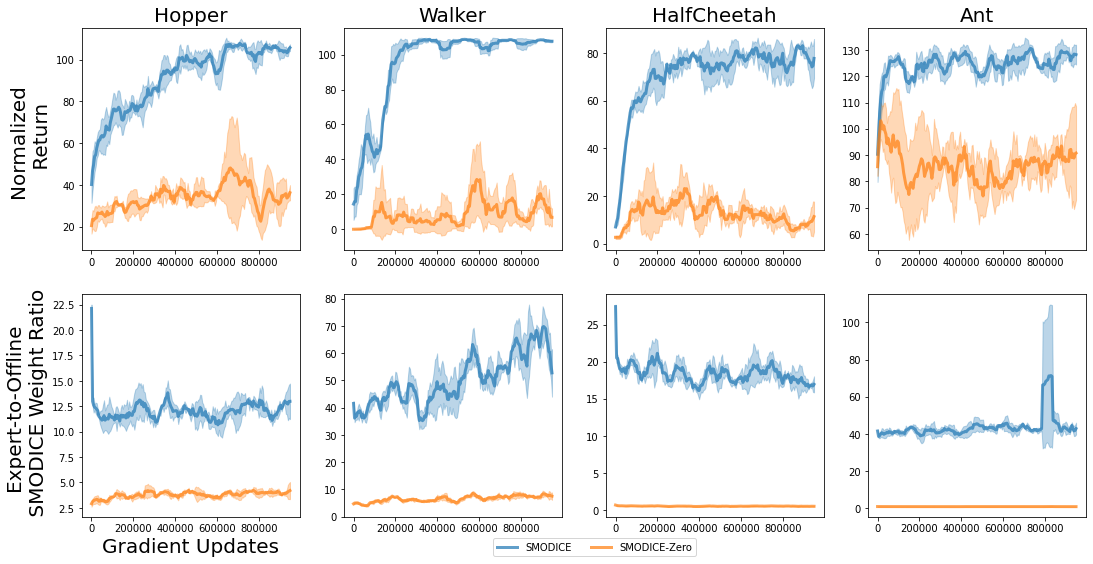}
\caption{SMODICE vs. SMODICE-Zero. Using the discriminator-based reward, SMODICE assigns much higher weights to expert-quality data.}
\label{figure:smodice-weight-ratio}
\end{figure}

\begin{figure}[t!]
\centering
\includegraphics[width=0.33\textwidth]{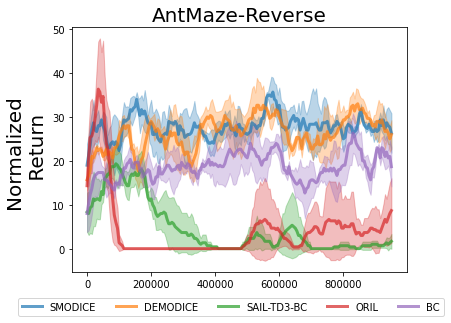}
\caption{Offline imitation learning results on AntMaze-Reverse. SMODICE is still among the best performing methods, while both SAIL-TD3-BC and ORIL collapse, demonstrating their sensitivity to the offline dataset composition.}
\label{figure:oil-ablations}
\end{figure}

\section{Offline IL from mismatched Expert Experimental Details}
\label{appendix:oilhe}
\subsection{Continuous Control Experiments} 
\para{Mismatched Experts.} The mismatched experts are illustrated in Figure \ref{figure:mismatched-expert}.

\begin{figure}[t!]
\centering
\resizebox{\textwidth}{!}{
\subfigure[HalfCheetah-Short]{\includegraphics[width=0.33\textwidth]{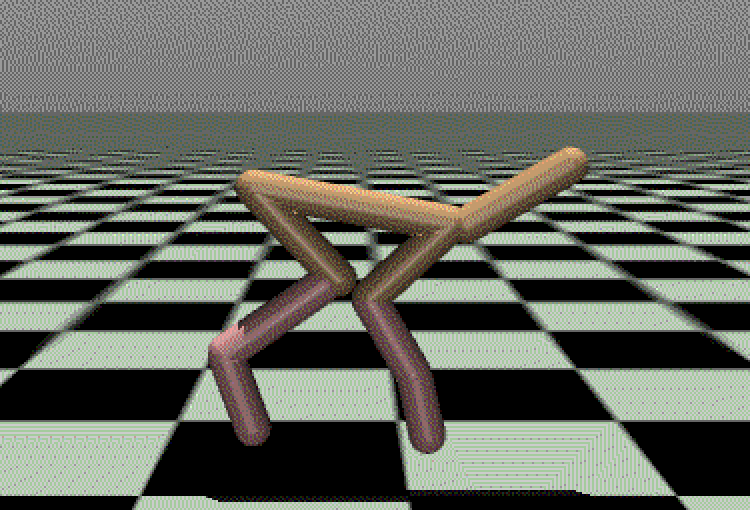}}
\subfigure[Ant-Disabled]{\includegraphics[width=0.33\textwidth]{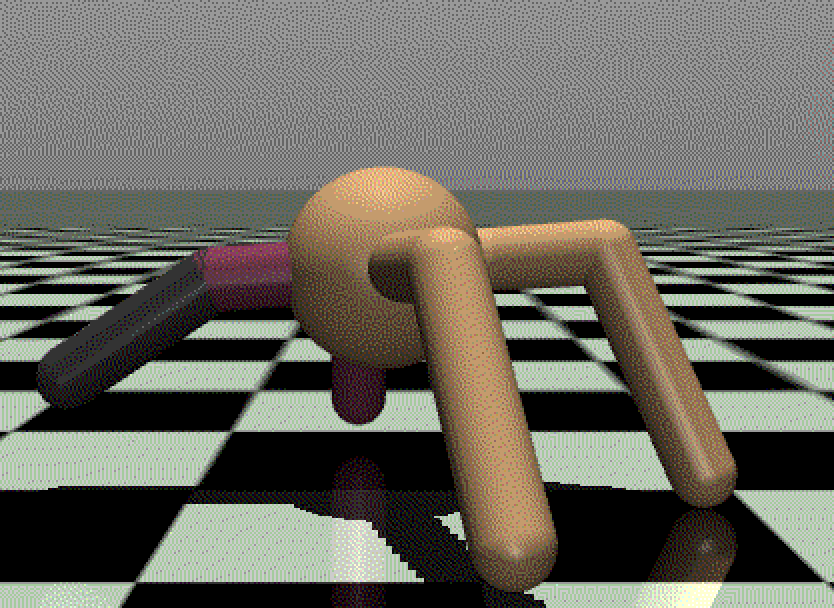}}
\subfigure[PointMass-Maze]{\includegraphics[width=0.33\textwidth]{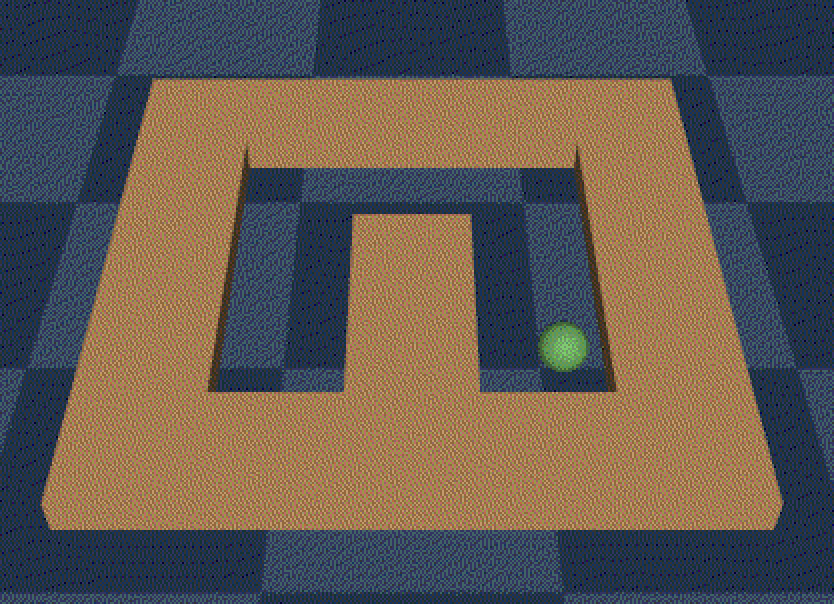}}}
\caption{Illustrations of the mismatched experts.}
\label{figure:mismatched-expert}
\end{figure}

\para{Comparison between PointMass and Ant experts for AntMaze.}
The trajectories of PointMass and Ant experts are illustrated in Figure \ref{figure:antmaze-expert-vis}. As can be seen, the PointMass trajectory is more regular and smooth due to its simpler dynamics and the use of a waypoint controller. In contrast, the ant trajectory is much less well-behaved because solving the maze task using the Ant agent is intrinsically a difficult task; consequently, it is difficult to provide an Ant demonstration. This example serves as a strong motivating problem for offline imitation learning with mismatched experts.

\begin{figure}[t!]
\centering
\subfigure[PointMass Expert]{\includegraphics[width=0.33\textwidth]{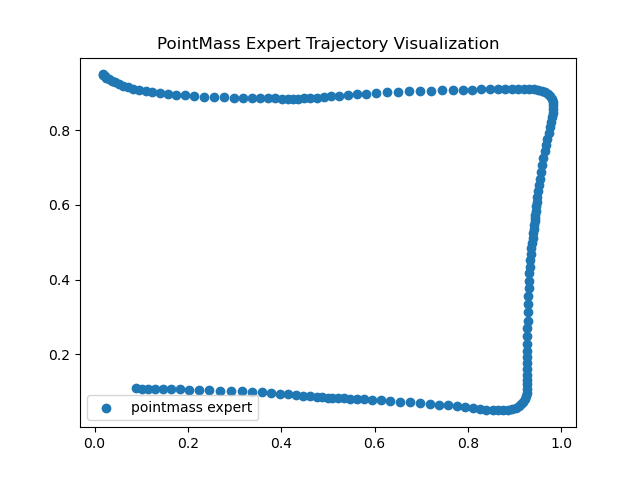}}
\subfigure[Ant Expert]{\includegraphics[width=0.33\textwidth]{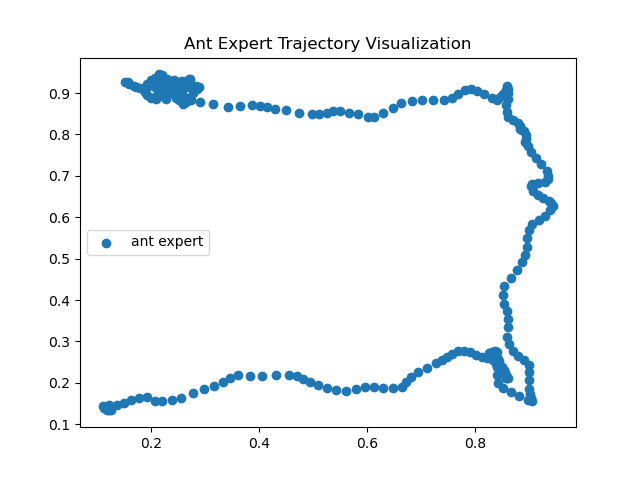}}
\caption{Trajectory visualizations of AntMaze experts.}
\label{figure:antmaze-expert-vis}
\end{figure}

\subsection{Quantitative Analysis of Figure \ref{figure:offline-il-observations-mismatched}}
\label{appendix:quantitative-analysis-oil-mismatched}
We quantitatively measure the percentage drop-in-performance for each method in Figure \ref{figure:offline-il-observations-mismatched}, computed as $\frac{|\text{max\_original}-\text{max\_mismatched}|}{\text{max\_original}}$. Note that this metric favors the baselines as taking the maximum value advantages methods that are more unstable. Nevertheless, as shown in Table \ref{table:oil-mismatched}, SMODICE is still by far the most robust method overall and in each individual task. As expected, ORIL does the worst as it is not designed to handle mismatched dynamics; this shows that using a state-based discriminator in itself is not sufficient.

\begin{table}[]
\centering 
\caption{Relative performance drop with mismatched experts.}
\resizebox{0.5\columnwidth}{!}{
\begin{tabular}{|c|c|c|c|c|}
\hline
Algorithm   & HalfCheetah & Ant    & AntMaze    & Average    \\ \hline
SMODICE     & \textbf{70.7\%}  & \textbf{3.3\%}  & \textbf{29.7\%} & \textbf{34.5\%} \\ \hline
SAIL-TD3-BC & 88.9\%      & 6.8\%  & 50.6\%     & 48.8\%     \\ \hline
ORIL        & 91.8\%      & 42.2\% & 72.7\%     & 68.9\%     \\ \hline
\end{tabular}}
\label{table:oil-mismatched}
\end{table}

\section{Offline IL from Examples Experimental Details} 
\label{appendix:oile} 
\subsection{Datasets}
We collect 300 success-state examples for each of the tasks. The examples are randomly sampled from the subset of the offline dataset that achieves the task. Task success is verfied through a pre-defined sparse reward function (e.g., distance threshold function). 

\subsection{Environments.}
\para{PointMass-4Direction.} This environment is adapted from the ``maze2d-umaze-v0'' environment in D4RL by changing the map configuration. The environment termination condition is triggered when the agent successfully comes within a small radius of the specified goal. 

\para{AntMaze-Example.} This environment is identical to the environments used in previous two settings.

\para{Kettle and Microwave.} These environments are adapted from the ''kitchen-mixed-v0'' environment in D4RL. The environments are identical as the original except the termination conditions. Both of these tasks terminate when the Franka robot places the specified object within a small radius of the desired configuration. 

\subsection{Examples of Success States}
All success states are extracted from the offline dataset used for policy training. We illustrate one representative example from each task in Figure \ref{figure:success-states}. 

\begin{figure}[t!]
\centering
\resizebox{\textwidth}{!}{
\subfigure[PointMass-4Direction]{\includegraphics[width=0.24\textwidth]{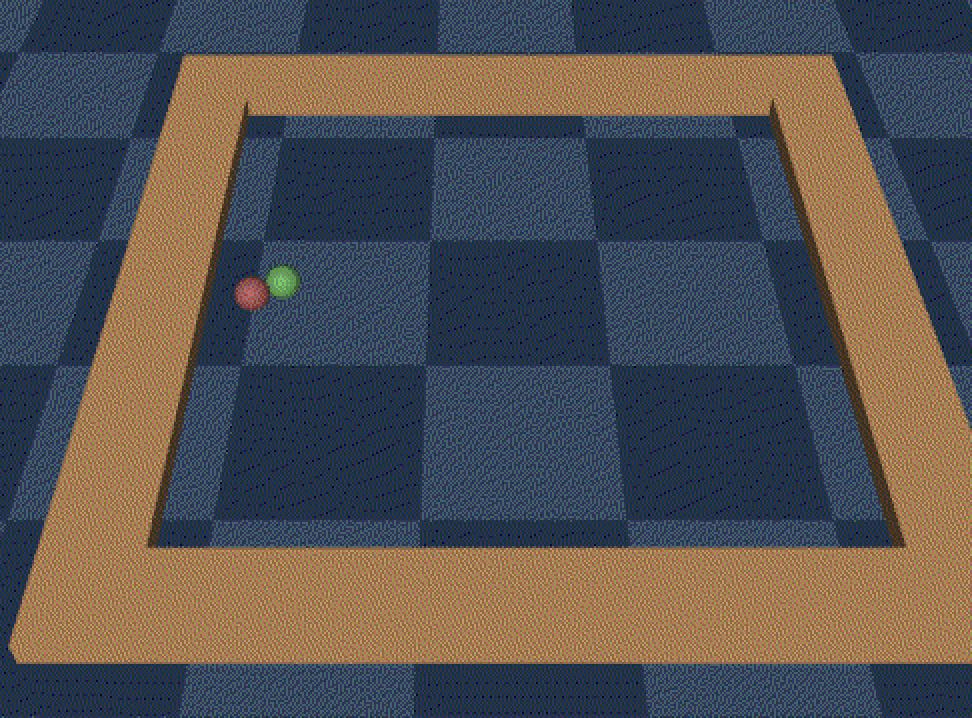}}
\subfigure[AntMaze]{\includegraphics[width=0.24\textwidth]{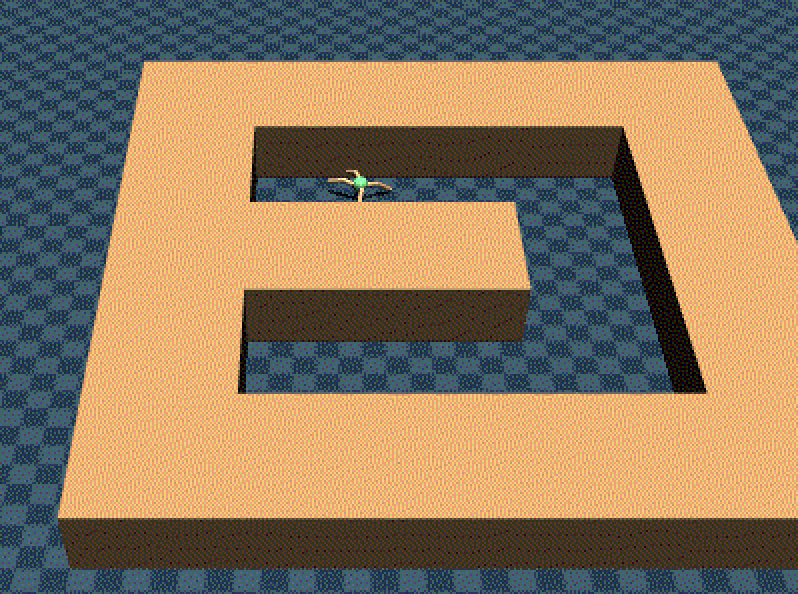}}
\subfigure[Kettle]{\includegraphics[width=0.24\textwidth]{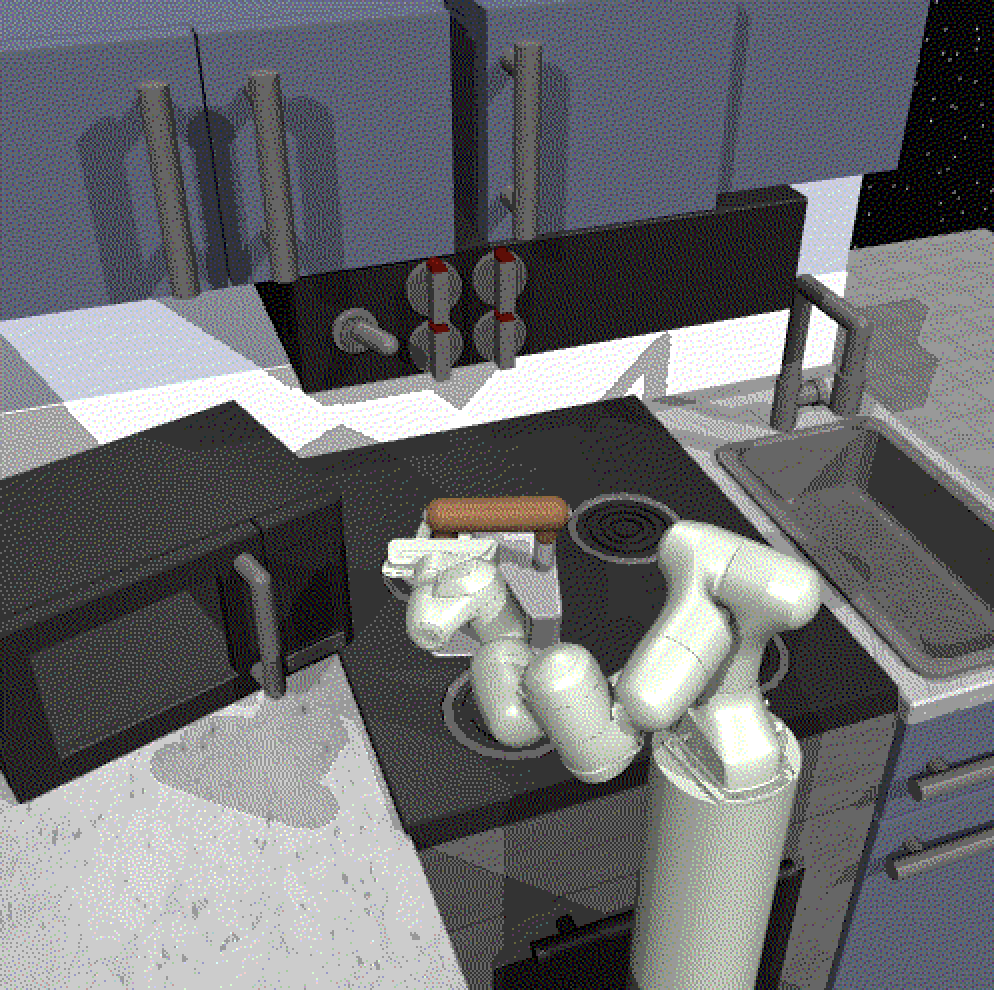}}
\subfigure[Microwave]{\includegraphics[width=0.24\textwidth]{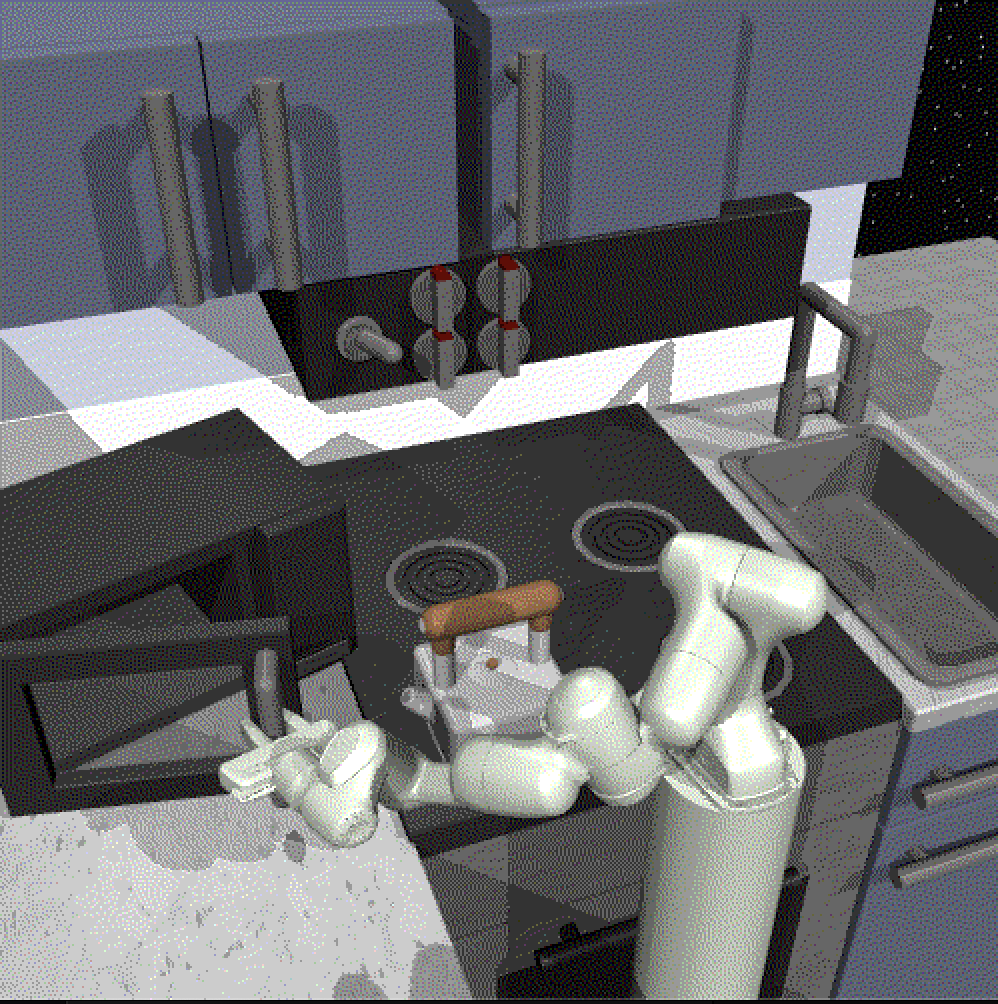}}}
\caption{Illustrations of success examples.}
\label{figure:success-states}
\end{figure}

\end{document}